\documentclass{article}

\usepackage[final,nonatbib]{arxiv}

\usepackage[utf8]{inputenc} 
\usepackage[T1]{fontenc}    
\usepackage{hyperref}       
\usepackage{url}            
\usepackage{booktabs}       
\usepackage{amsfonts}       
\usepackage{adjustbox}      
\usepackage{nicefrac}       
\usepackage{microtype}      
\usepackage{xcolor}         
\usepackage{pifont}     

\definecolor{DarkGreen}{RGB}{0,128,0}   
\definecolor{DarkRed}{RGB}{139,0,0}     

\usepackage{pifont}
\newcommand{\cmark}{\textcolor{DarkGreen}{\ding{51}}} 
\newcommand{\xmark}{\textcolor{DarkRed}{\ding{55}}}   

\usepackage{amsmath}
\usepackage{amssymb}
\usepackage{mathtools}
\usepackage{amsthm}
\usepackage{enumitem}
\usepackage{amssymb}
\usepackage{wrapfig}   

\usepackage{graphicx}
\usepackage{booktabs}
\usepackage{parskip}
\usepackage{overpic}
\usepackage{subcaption}
\usepackage{empheq}
\usepackage{authblk}
\usepackage{moresize}
\usepackage{arydshln}
\usepackage{multirow}
\usepackage{enumitem}

\usepackage{algorithm}
\usepackage{algorithmic}

\usepackage{tikz}
\usetikzlibrary{arrows.meta, positioning, shapes.geometric}
\usetikzlibrary{calc}

\usepackage{hyperref}
\hypersetup{
    colorlinks,
    linkcolor={blue!90!black},
    citecolor={magenta},
    urlcolor={blue!80!black}
}

\newtheorem{prop}{Proposition}
\newtheorem{lem}{Lemma}

\usepackage[defaultlines=3,all]{nowidow}   

\title{FLEX: A Backbone for Diffusion-Based Modeling of Spatio-temporal Physical Systems}

\usepackage{authblk}

\author[1,2]{N. Benjamin Erichson\thanks{Equal contribution. Corresponding author: erichson@lbl.gov.}$\,\,\,$}
\author[1]{Vinicius Mikuni$^*$}
\author[3]{Dongwei Lyu}
\author[2]{Yang Gao}
\author[4]{Omri Azencot}
\author[5]{\\Soon Hoe Lim}
\author[1,2,3]{Michael W. Mahoney}

\affil[1]{Lawrence Berkeley National Laboratory}
\affil[2]{International Computer Science Institute}
\affil[3]{University of California, Berkeley}
\affil[4]{Ben-Gurion University of the Negev}
\affil[5]{KTH Royal Institute of Technology}

\begin{document}

\maketitle

\begin{abstract}
We introduce \textsc{FLEX} (FLow EXpert), a backbone architecture for generative modeling of spatio-temporal physical systems using diffusion models. FLEX operates in the residual space rather than on raw data, a modeling choice that we motivate theoretically, showing that it reduces the variance of the velocity field in the diffusion model, which helps stabilize training. FLEX integrates a latent Transformer into a U-Net with standard convolutional ResNet layers and incorporates a redesigned skip connection scheme. This hybrid design enables the model to capture both local spatial detail and long-range dependencies in latent space. To improve spatio-temporal conditioning, FLEX uses a task-specific encoder that processes auxiliary inputs such as coarse or past snapshots. Weak conditioning is applied to the shared encoder via skip connections to promote generalization, while strong conditioning is applied to the decoder through both skip and bottleneck features to ensure reconstruction fidelity. FLEX achieves accurate predictions for super-resolution and forecasting tasks using as few as two reverse diffusion steps. It also produces calibrated uncertainty estimates through sampling. Evaluations on high-resolution 2D turbulence data show that FLEX outperforms strong baselines and generalizes to out-of-distribution settings, including unseen Reynolds numbers, physical observables (e.g., fluid flow velocity fields), and boundary conditions.
\end{abstract}

\section{Introduction}\label{sec:intro}

Diffusion models have achieved remarkable success in image~\cite{rombach2022high, esser2024scaling} and video~\cite{ho2022video,blattmann2023stable} generation and are now being adopted in scientific research. Unlike discriminative models that directly map inputs to outputs, diffusion models generate samples by progressively denoising random noise~\cite{sohl2015deep,ho2020denoising,song2021denoising}. This allows them to better model both low- and high-frequency structures, which is essential for preserving physical properties. Their probabilistic nature also supports the generation of multiple plausible outputs, providing a foundation for uncertainty quantification.
Conditioned on auxiliary inputs (e.g., coarse spatial maps or past states), diffusion models excel in spatio-temporal prediction, enabling super-resolution~\cite{molinaro2024generative,wan2024statistical} and forecasting~\cite{price2023gencast,pathak2024kilometer,kohl2023turbulent,luo2024difffluid,oommen2024integrating} in climate and fluid mechanics.
However, accurate modeling of high-dimensional chaotic systems remains a challenge~\cite{cheng2025gradientfree}.

Most spatio-temporal diffusion models for science (see App.~\ref{app_relatedwork}) adopt convolutional U-Net backbones, following the variant used in image denoising diffusion models~\cite{ho2020denoising}. While convolutions provide strong inductive biases for modeling local structure and edge statistics, they struggle to capture long-range dependencies. The receptive field grows only with depth, requiring spatially or temporally distant information to propagate through many layers. To partially address this, recent work augments U-Nets with multi-head self-attention at coarse resolutions~\cite{nichol2021improved}, enabling global context aggregation. Self-attention constructs a similarity matrix, giving each token a one-shot view of the entire feature map and improving alignment across large spatial scales. However, these augmentations typically omit the feed-forward layers found in standard Transformer blocks, limiting channel-wise mixing. As a result, attention-augmented U-Nets can identify whether distant positions should interact, but they often lack the expressiveness to encode those interactions into globally coherent latent representations. This limits performance in tasks where predictive accuracy relies on multi-scale consistency, such as turbulent fluid flows.
A second limitation lies in how conditional information is injected into the model. The standard approach concatenates auxiliary inputs with the model input along the channel dimension, introducing a strong, spatially aligned inductive bias. Rather than learning the low-frequency content, the model can simply leverage the provided information and focus on learning residual high-frequency details. This can restrict the model’s posterior, reducing sample diversity.

\begin{figure}[!t]
	\centering
	\begin{overpic}[width=0.95\linewidth]
		{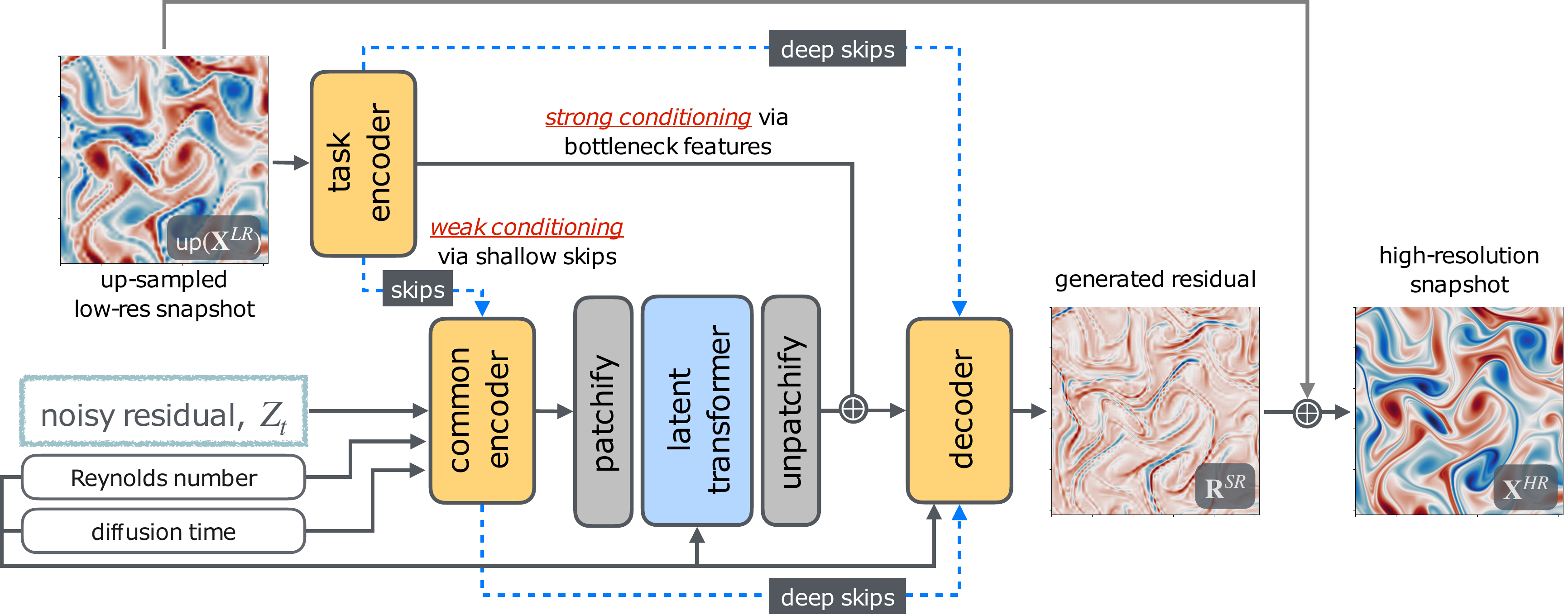}
	\end{overpic}
	\caption{\textsc{FLEX} is a backbone for modeling spatio-temporal physical systems using diffusion models. It learns residual corrections conditioned on task-specific inputs (e.g., low-resolution or past states) and physical parameters. The architecture integrates a task-specific encoder, a common encoder, a Transformer operating in latent space, and a decoder within a U-Net-style framework. The task-specific encoder weakly conditions the common encoder via shallow skip connections, and strongly conditions the decoder through deep skip connections and an embedding of the full conditional input. Shown here: FLEX instantiated for super-resolution.}
	\label{fig:flex_sr_overview}
\end{figure}

To address these limitations, we introduce \textsc{FLEX} (FLow EXpert), a new backbone architecture for diffusion models that combines the locality of convolutional operations with the global modeling capacity of Vision Transformers (ViT)~\cite{dosovitskiy2021an}. In addition, we propose to operate FLEX in residual space, and we adopt a velocity-based parameterization to improve performance. We provide a theoretical analysis (see App.~\ref{app_theory_velocitymodel}) to motivate these design choices, showing that initializing the velocity-parameterized model with residual samples reduces the variance of the target velocity field, which can stabilize training and improve sample quality.

FLEX replaces the middle block of the U-Net with a ViT that operates on patch size 1, ensuring each token corresponds to a single spatial location, while enabling all-to-all communication via self-attention. This design captures global dependencies without sacrificing spatial resolution. A redesigned skip-connection scheme integrates the ViT bottleneck with standard ResNet-style convolutional blocks, supporting fine-scale feature reconstruction and long-range consistency. Unlike the hybrid model proposed by~\cite{hoogeboom2023simple,hoogeboom2024simpler}, FLEX retains more residual connections at each resolution scale, which we find helps improve performance. Conditioning variables, such as the Reynolds number, are injected into the ViT bottleneck as learnable tokens.
To improve conditioning, FLEX introduces a task-specific encoder that processes auxiliary inputs (e.g., low-resolution snapshots or past states) separately from the main generative path. This enables a hierarchical conditioning strategy in which we differentiate between two levels of conditioning strength:
(i) \emph{Weak conditioning} injects only partial information about the conditioning variables into the shared encoder via shallow skip connections at early layers. These features provide coarse structural cues while preserving uncertainty, allowing the latent space to remain largely task-agnostic, and enabling diverse generative outcomes.
(ii) \emph{Strong conditioning}, in contrast, involves the explicit injection of the latent embedding of the full conditioning signal in addition to the deep skip connections. 
This provides strong guidance in aligning the outcomes with the conditioning variables. 

To evaluate FLEX, we use SuperBench's 2D incompressible Navier–Stokes–Kraichnan turbulence (NSKT) dataset~\cite{ren2025superbench}, which spans Reynolds numbers from $1{,}000$ to $36{,}000$ and captures increasingly chaotic multiscale fluid dynamics at resolution $2048\times 2048$. We assess performance under two generalization settings: (i) unseen initial conditions at trained Reynolds numbers; and (ii) unseen Reynolds numbers. Compared to other baselines, FLEX outperforms both convolutional U-Net backbones~\cite{nichol2021improved} and ViT variants~\cite{bao2023all} on super-resolution and forecasting tasks.
FLEX also demonstrates zero-shot generalization across physical observables and boundary conditions. Although trained only on vorticity fields with periodic boundaries, it generalizes to velocity fields and successfully transfers to PDEBench's inhomogeneous Navier–Stokes data with Dirichlet boundaries~\cite{takamoto2022pdebench}. Interestingly, FLEX achieves better zero-shot forecasts on PDEBench than on the in-distribution test cases, suggesting that (popular) lower-fidelity simulations may present deceptively easier generalization challenges.

In summary, our main contributions are as follows:
\begin{itemize}[leftmargin=*]
	
	\item We address the limited global modeling capacity of U-Net backbones by introducing a ViT bottleneck into the diffusion model architecture. The ViT operates on patch size 1, providing each spatial location with a global receptive field, while preserving spatial resolution. This module replaces the middle block of the U-Net, and it is integrated with convolutional layers through skip connections, enabling long-range dependencies to be modeled in latent space.
	
	\item To overcome the limitations of direct input concatenation for conditioning, we introduce a task-specific encoder that processes auxiliary inputs such as coarse snapshots or past time steps. This encoder enables hierarchical conditioning: weak conditioning is applied to the shared encoder via skip connections to promote generalization and diversity; while strong conditioning is applied to the decoder through both skip and bottleneck features to ensure reconstruction fidelity.
	
	\item We use a diffusion model formulation based on a velocity parameterization in the residual space and provide theoretical analysis showing that this reduces the variance of the optimal velocity field, which can improve training stability and sample quality (see Sec. \ref{sec_diffusion_residual}, Prop. \ref{prop_1}–\ref{prop_2}, and App.~\ref{app_theory_velocitymodel}).
	
	\item We evaluate FLEX on high-resolution ($2048\times 2048$) 2D turbulence simulations. FLEX achieves state-of-the-art performance on super-resolution and forecasting tasks, outperforming baselines, and producing calibrated uncertainty estimates via diffusion-based ensembles. FLEX also demonstrates zero-shot generalization to out-of-distribution settings, including unseen boundary conditions.
\end{itemize}

\section{The FLEX Architecture for Diffusion-Based Spatio-temporal Modeling}
\label{sec:flex}

\textsc{FLEX} serves as a drop-in replacement for the U-Net backbone commonly-used in diffusion models, designed to address the limitations of the standard U-Net and ViT backbones identified in Section~\ref{sec:intro}: 
\begin{enumerate}[label=(\roman*), leftmargin=*]
	\item \textbf{Global dependency modeling:} FLEX incorporates a ViT bottleneck, including both self-attention and MLP layers, enabling each spatial token to access global context and re-encode interactions into a coherent latent representation. 
	
	\item \textbf{Scalability:} For ViTs, the expressiveness of the model is strongly related to the dimensionality of the input patches, often requiring smaller patch sizes to achieve good performance. For high-resolution images, small patch sizes are prohibitive both in terms of memory and number of floating point operations (FLOPs). FLEX is able to reduce the overall number of patches without compromising the performance, due to the additional skip connections. 
	
	\item \textbf{Flexible conditioning:} Rather than concatenating conditioning inputs directly with the noised residuals, FLEX routes task-specific information through a separate encoder and injects it hierarchically via skip connections. This avoids introducing strong inductive bias.
\end{enumerate}

In the following, we discuss a multi-task version of FLEX, illustrated in Figure~\ref{fig:overview}, which includes dedicated encoder pathways for super-resolution and forecasting tasks. A single-task version can be derived by removing one of the task-specific branches and its associated skip connections and bottleneck features.
Given a noisy residual $Z_t$ at diffusion time $t$, a low-resolution input snapshot ${X}_n^{\text{LR}}$, a high-resolution current snapshot ${X}_n^{\text{HR}}$, and conditioning information $C$, the goal is to predict the residual velocity field $\hat{v} = v_\theta(t, Z_t, C)$. FLEX is composed of the following components:

\begin{figure}[!t]
	\centering
	\begin{overpic}[width=0.95\linewidth]
		{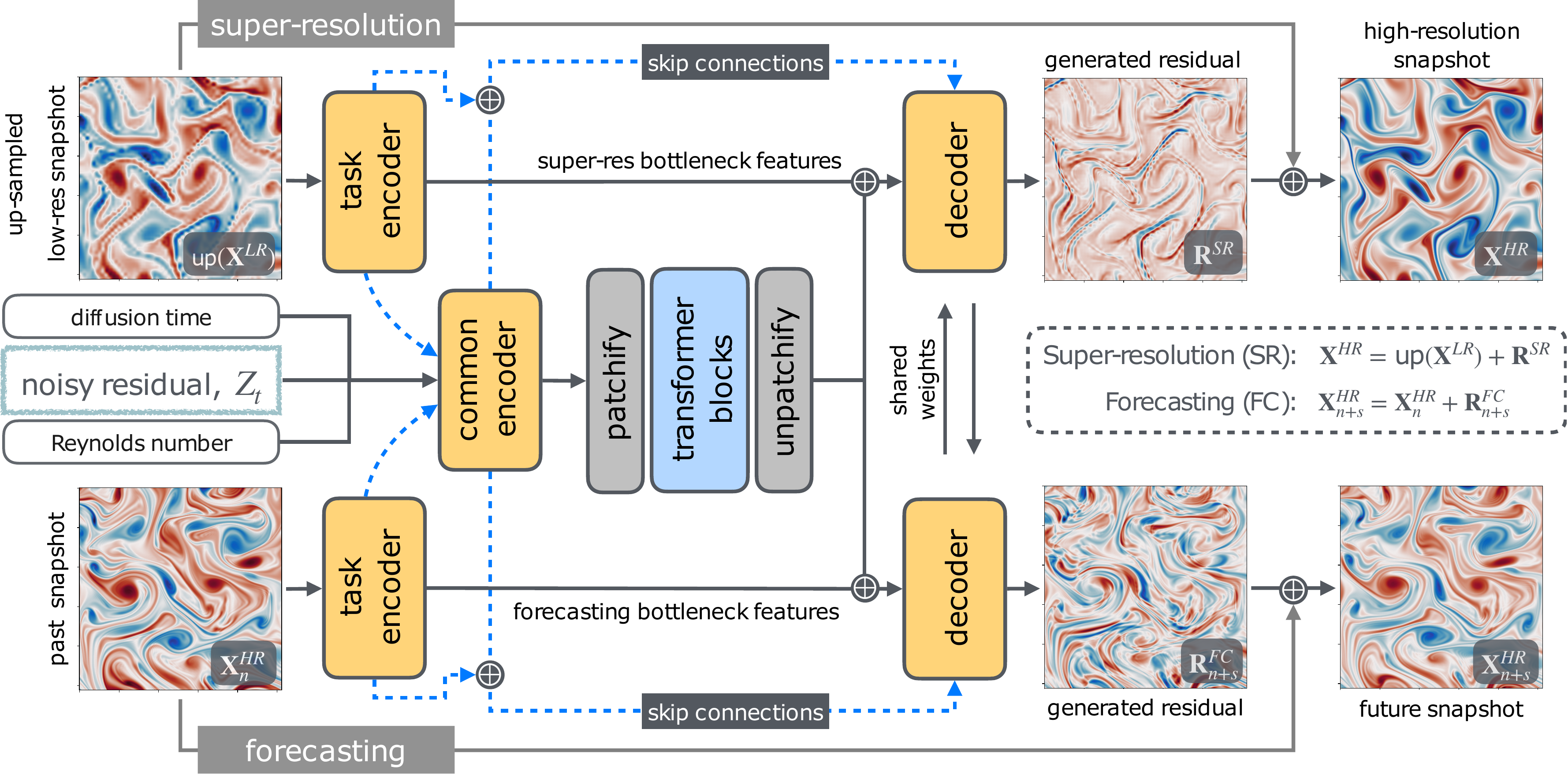}
	\end{overpic}
	\caption{Illustration of the multi-task \textsc{FLEX} backbone instantiated for super-resolution and forecasting. The backbone includes two task-specific encoders and shares a common encoder, a latent Transformer, and a decoder. During inference, only the encoder corresponding to the target task is used.}
	\label{fig:overview}
\end{figure}

\begin{itemize}[leftmargin=*]
	
	\item \textbf{The task-specific convolutional encoders} $\mathcal E_{\text{SR}}$ and $\mathcal E_{\text{FC}}$ process the conditional inputs for super-resolution and forecasting, respectively. Each encoder takes as input the relevant snapshot, $\texttt{up}(X_n^{\text{LR}})$ or $X_n^{\text{HR}}$, along with $C$. It outputs (i) a bottleneck feature tensor $h_{\mathrm{task}} \in \mathbb{R}^{C \times \frac{H}{2^L} \times \frac{W}{2^L}}$ and (ii) a set of multi-scale skip tensors $\{ s_{\mathrm{task}}^{(l)} \}_{l=0}^{L}$, where $C$ is the number of channels, $H \times W$ is the input resolution, and $L$ is the number of downsampling levels. When training with both tasks active, we update the general representation of the network twice per batch, once for each task.  During inference, only the encoder for the target task is used.
	
	\item \textbf{The shared convolutional encoder} $\mathcal E_{\text{common}}$ processes the noisy residual $Z_t$, conditioned on $(t, C)$, and uses task skip connections from the dedicated encoders. In the case of multitask training, task-dependent skip connections are combined in the batch dimension as $s^{(l)}_{\text{fused}} = \operatorname{vstack}(s^{(l)}_{\text{SR}}, s^{(l)}_{\text{FC}})$ at each resolution level. This encoder outputs a latent bottleneck feature $h_{\text{common}}$ and a corresponding set of skip connections $\{ s_{\text{common}}^{(l)} \}$.
	
	\item \textbf{The latent Transformer} $\mathcal T$ replaces the middle block of the U-Net to provide global spatial context and full channel mixing. The bottleneck $h_{\text{common}}$ is patchified with patch size 1, yielding $N = \frac{H}{2^L} \cdot \frac{W}{2^L}$ tokens of dimension $d$, forming $X \in \mathbb{R}^{B \times N \times d}$, where $B$ is the batch size. After adding the learned 2D positional embeddings, we prepend a conditioning token $e_{C, t} = \mathrm{MLP} \bigl[ \phi_{\text{fourier}}(t)\;\Vert\;C \bigr]$, and pass the full sequence through the latent Transformer blocks.
	Each block follows pre-norm ordering:
	\begin{equation}
		{Y} = {X} + \mathrm{SA}(\mathrm{LN}({X})), \quad
		{X}' = {Y} + \mathrm{MLP}(\mathrm{LN}({Y})).
	\end{equation}
	This design enables each spatial token to access global information and transform it with expressive, non-linear operations, capabilities that are missing from attention-augmented U-Nets. 
	
	\item \textbf{The decoder} $\mathcal G$ upsamples the latent representation back to the output resolution. The latent Transformer output enters the decoder blocks of the U-Net, combining the information of the shared and task-specific encoders through skip connections.  At each resolution level $l$, the decoder applies:
	\begin{equation}
		h^{(l)} = g^{(l)}\left( h^{(l+1)}, s_{\text{common}}^{(l)} + s_{\text{fused}}^{(l)}, C, t \right),
	\end{equation}
	where $g^{(l)}$ is a convolutional decoder block.
\end{itemize}

\paragraph{Information Flow.}

FLEX implements a hierarchical conditioning scheme by controlling the depth and location at which auxiliary information is injected into the network. We refer to the two levels of conditioning strength as \emph{weak} and \emph{strong} conditioning. 

\begin{itemize}[leftmargin=*]
	\item \textbf{Weak conditioning} is applied within the shared encoder $\mathcal{E}_{\text{common}}$ by injecting shallow skip connections from task-specific encoders $\mathcal{E}_{\text{SR}}$ and $\mathcal{E}_{\text{FC}}$ at early resolution levels $l \leq L_{\text{weak}}$. These skip tensors $s^{(l)}_{\mathrm{task}} \in \mathbb{R}^{C_l \times H_l \times W_l}$ are added elementwise to the corresponding activations in $\mathcal{E}_{\text{common}}$, where $C_l$, $H_l$, and $W_l$ denote the number of channels and spatial dimensions at level $l$. This shallow injection encodes coarse semantic features of the conditioning variable (e.g., global flow structure), while withholding fine-grained details. As a result, the encoder is encouraged to form a task-agnostic latent representation, preserving posterior variance and enabling diverse sample generation across repeated stochastic passes.
	
	\item  \textbf{Strong conditioning} is applied downstream in the decoder $\mathcal{G}$, where the model integrates richer forms of auxiliary information. At each resolution level, we concatenate the skip tensors of the shared encoder $s^{(l)}_{\text{common}}$ and the task-specific encoder $s^{(l)}_{\text{task}}$ along the channel dimension. This enables the decoder to incorporate fine-scale conditioning signals. 
\end{itemize}

This two-leveled conditioning scheme allows FLEX to disentangle task-invariant feature learning from task-specific guidance, effectively balancing reconstruction accuracy and posterior diversity. Empirically, we find that weak conditioning alone results in higher sample variance and underfitting, while strong conditioning alone leads to posterior collapse and overconfident predictions. The combination of both achieves the best trade-off, yielding calibrated uncertainty estimates and high-fidelity reconstructions across forecasting and super-resolution tasks.

\section{Training FLEX}
\label{sec_diffusion_residual}

We focus on modeling the residual distribution rather than the full data distribution. Residuals, defined as correction terms, capture the difference between the low-resolution input (or past state) and the desired high-resolution (or future state). 
Specifically, we define the residuals for super-resolution as ${R}^{\text{SR}}_{n} = {X}_n^{\text{HR}} - \texttt{up}({X}_n^{\text{LR}})$,
where $\texttt{up}(\cdot)$ denotes an upsampling operation (e.g., bicubic interpolation). Similarly, we define the residuals for the forecasting tasks as ${R}^{\text{FC}}_{n+s} = {X}_{n+s}^{\text{HR}} - {X}_n^{\text{HR}}$, where $s$ is the forecast horizon. 
That is, in super-resolution, the residual corresponds to the difference between a high-resolution snapshot ${X}_n^{\text{HR}}$ and its upsampled low-resolution counterpart; while in forecasting, it represents the evolution from the current state ${X}_n^{\text{HR}}$ to a future state ${X}_{n+s}^{\text{HR}}$. While residual modeling for super-resolution was introduced by SRDiff~\cite{li2022srdiff}, we extend this idea to forecasting, interpreting the residual as an update term that advances the system forward in time. This formulation also generalizes to related tasks such as flow reconstruction from sparse sensor data~\cite{erichson2020shallow}.

\begin{figure}[!b]
	\centering
	\begin{overpic}[width=0.96\linewidth]
		{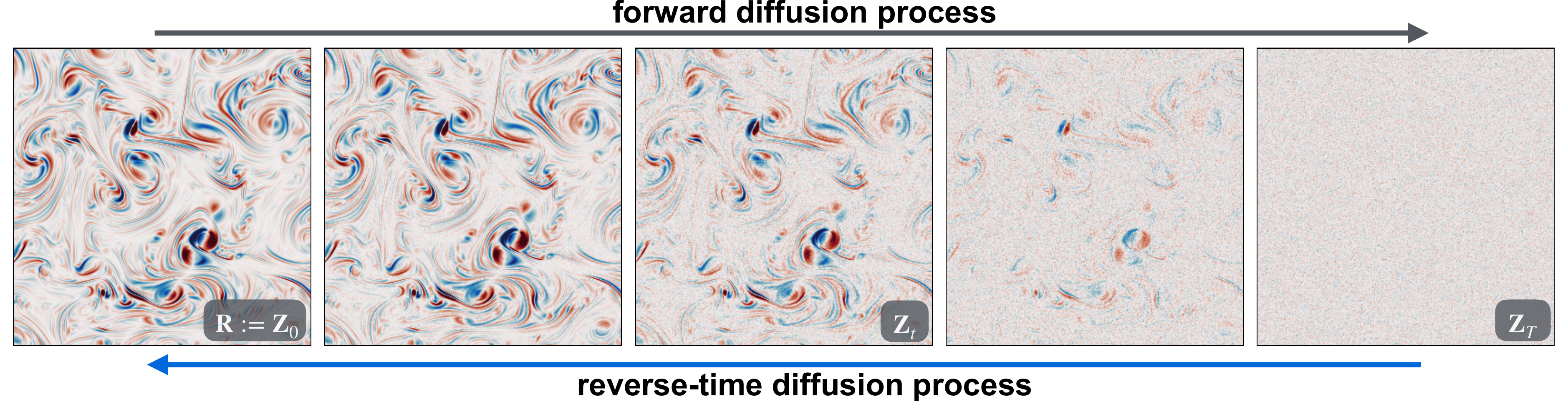}
	\end{overpic}\vspace{-0.0cm}
	\caption{Illustration of the score-based diffusion process. During training, residual samples are progressively corrupted by Gaussian noise via the forward process according to a noise schedule. During inference, samples are reconstructed by simulating the reverse-time SDE/ODE, which progressively denoises a noisy sample back toward the clean residual. In our framework, we use the velocity-based parameterization (Equation~\eqref{eq_velocity}) of the score (see also Equation~\eqref{eq_v_intermsof_score} in Appendix~\ref{app_theory_velocitymodel}).}
	\label{fig:DDPM_illustration}
\end{figure}

\subsection{Residual and Velocity Parametrization}

We model the conditional residual distribution $p_\theta({R} | C)$ using a continuous-time  diffusion process~\cite{song2021score}, where $C$ denotes the conditioning information. 
The forward diffusion process progressively perturbs a clean residual $R$ over time with Gaussian noise (see Figure~\ref{fig:DDPM_illustration}):
\begin{equation}
	Z_t = \alpha(t) R + \sigma(t) \epsilon, \quad \epsilon \sim \mathcal{N}(0, I),
\end{equation}
where $\alpha(t)$ and $\sigma(t)$ are time-dependent coefficients and $t \in [0,1]$ denotes the diffusion time. These coefficients describe the noise schedule and are chosen so that at $t=0$, $Z_0$ models the clean residual $R$, while at $t=1$, $Z_1$ realizes samples from an isotropic Gaussian distribution. 

New residual samples are generated by simulating the reverse-time dynamics, described in terms of the typically unknown \emph{score function}, defined as the gradient of the log-probability density of the perturbed residual, $\nabla_{Z_t} \log p_t(Z_t | C)$. The reverse-time evolution can be described by either an ordinary differential equation (ODE) or stochastic differential equations (SDEs) (see App. \ref{app_theory_velocitymodel} for background and details). Here we consider the Probability Flow ODE:
\begin{equation} \label{eq_pfode}
	\frac{d\tilde{Z}_t}{dt} = - f(1-t) \tilde{Z}_t  + \frac{g^2(1-t)}{2} \nabla_{\tilde{Z}_t} \log p_{1-t}(\tilde{Z}_t | C),
\end{equation}
where $f(t) = \frac{\dot{\alpha}(t)}{\alpha(t)}$ and $g(t) = 2 \sigma^2(t) \left( \frac{\dot{\sigma}(t)}{\sigma(t)} - \frac{\dot{\alpha}(t)}{\alpha(t)} \right)$. Throughout, we use the convention that reverse-time variables depend on $1-t$ to emphasize the backward temporal evolution. 

Rather than directly estimating the score $\nabla_{Z_t} \log p_t(Z_t | C)$, we estimate a velocity field, adopting the velocity-based parameterization introduced by~\cite{salimans2022progressive}. The target velocity field is defined as:
\begin{equation} \label{eq_velocity}
	v(t, R) = \alpha(t) \epsilon - \sigma(t) R,
\end{equation}
which linearly combines noise realization $\epsilon$ and clean residual $R$. 
We learn the velocity field using a neural network $v_\theta(t, Z_t, C)$; see Sec.~\ref{sec:flex} for details. The network is trained to minimize the $L_2$ loss between the predicted and the target velocity field:
\begin{equation}
	\mathcal{L}_{\text{vel}} = \mathbb{E}_{t, R, \epsilon} \left[ \left\| v(t, R) - v_\theta(t, Z_t, C) \right\|_2^2 \right],
\end{equation}
where the expectation is taken over randomly sampled diffusion times,  residual samples,  and noise samples. The full training procedure is summarized in Algorithm~\ref{alg:training}. Specific details on the diffusion model (the choice of $\alpha(t)$, $\sigma(t)$ and $C$) and training are provided in App. \ref{sec:model_details}-\ref{app:arch}.

\textbf{Theory.} In App. \ref{app_theory_velocitymodel}, we provide theoretical insight on the velocity parameterization and elucidate the benefits of training this diffusion model in the residual space instead of the raw data space (see Proposition \ref{prop_1}-\ref{prop_2} and the related discussion). In particular, we derive the optimal velocity field, and we show that using residual instead of raw data in the velocity-parameterized diffusion model helps reduce the variance of the optimal velocity field, which thus could stabilize training. 

\subsection{Inference and Ensemble Generation}
At inference time, we discretize the reverse-time dynamics to generate samples. Specifically, we use the Denoising Diffusion Implicit Models (DDIM) sampling procedure~\cite{song2021denoising} to perform deterministic reverse-time updates without stochastic noise injection. DDIM has been shown to be Euler’s method applied to reparameterization of ODE (\ref{eq_pfode}). Given a discretized time schedule $\{ t_0, t_1, \ldots, t_N \}$ with $t_0 = 1$ and $t_N = 0$, the DDIM update at step $i$ is given by:
\begin{equation}
	Z_{t_{i-1}} = \alpha(t_{i-1}) \left( \frac{Z_{t_i} - \sigma(t_i) v_\theta(t_i, Z_{t_i}, C)}{\alpha(t_i)} \right) + \sigma(t_{i-1}) v_\theta(t_i, Z_{t_i}, C),
\end{equation}
where $v_\theta(t_i, Z_{t_i}, C)$ denotes the predicted velocity at time $t_i$. This deterministic update maps $Z_{t_i}$ to $Z_{t_{i-1}}$, progressively denoising the sample toward the clean residual $R$, summarized in Algorithm~\ref{alg:sampling}.

We can generate an ensemble by repeatedly sampling from the the model using the same conditioning input but different random noise initializations. Each sample follows the reverse-time process based on DDIM, and the final prediction is obtained by averaging multiple runs. We compute the standard deviation among the members of the ensemble to quantify spatial uncertainty.

\section{Experimental Results}
\label{sec:results}

We evaluated \textsc{FLEX} on super-resolution and forecasting tasks using NSKT fluid flows proposed in SuperBench~\cite{ren2025superbench}; see App.~\ref{sec:data_details} for details. We use simulations with Reynolds numbers $Re = \{2\mathrm{k}, 4\mathrm{k}, 8\mathrm{k}, 16\mathrm{k}, 32\mathrm{k}\}$, consisting of $1{,}500$ snapshots, sampled at intervals of $\Delta t = 0.0005$ seconds. For training, we use only vorticity fields and, to manage memory constraints, training is performed on $256 \times 256$ patches, which we then stitch together to obtain the full-resolution $2048 \times 2048$ snapshots.
We evaluated generalization in two domain-shifted settings: new initial conditions at (i) seen Reynolds numbers; and (ii) unseen Reynolds numbers $Re = \{1\mathrm{k}, 12\mathrm{k}, 24\mathrm{k}, 36\mathrm{k}\}$.

\subsection{Super-Resolution of Turbulent Flows}
\label{sec:results_sr}

Following~\cite{ren2025superbench}, our super-resolution setting uses low-resolution inputs obtained by uniformly downsampling the high-resolution ground truth fields by a factor of $8$. This setup creates a challenging benchmark and allows for controlled model comparisons.
Table~\ref{tab:superres_vorticity} reports the average relative Frobenius norm error (RFNE) across seen and unseen Reynolds numbers. Our super-resolution model, FLEX-SR, consistently outperforms Transformer-based baselines such as SwinIR~\cite{liang2021swinir} and HAT~\cite{chen2023activating}, with the performance gap increasing at higher Reynolds numbers where the flow exhibits more complex multiscale structure. Among U-Nets with attention~\cite{nichol2021improved}, residual modeling improves performance particularly in the high Reynolds number regime. Diffusion models using U-Net~\cite{nichol2021improved} and ViT~\cite{bao2023all} backbones, while trained similarly to FLEX, outperform the Transformer baselines but still underperform relative to FLEX. In particular, the U-ViT backbone alone struggles, as we use a patch size of 8, the smallest number we could afford to obtain training times similar to FLEX and standard U-Net models.
The medium (M) and large (L) multitask (MT) models, FLEX-MT-M/L, trained jointly on super-resolution and forecasting, outperform FLEX-SR overall.
Note, that we use only 2 diffusion steps.

\begin{table}[!t]
	\vspace{-0.5cm}
	\caption{Relative reconstruction error (RFNE $\times 100$, lower is better) for various model architectures across different Reynolds numbers. Asterisks (*) indicate test Reynolds numbers not seen during training. Boldface marks the best performance, and underlined italics denote the second-best result.}\vspace{+0.1cm}
	\label{tab:superres_vorticity}
	\centering
	\scalebox{0.85}{
		\begin{tabular}{l r | cccccccc|c}
			\toprule
			\multicolumn{2}{c|}{Model} & 1k$^*$ & 4k & 8k & 12k$^*$ & 16k & 24k$^*$ & 32k & 36k$^*$ & Avg \\ 
			\midrule
			SwinIR-S  & 8M   & 0.5 & 1.5 & 3.4 & 5.4 & 7.4 & 11.8 & 15.6 & 17.2 & 7.8 \\  
			SwinIR-M  & 40M  & 0.4 & 1.5 & 3.6 & 5.7 & 7.9 & 12.6 & 16.5 & 18.3 & 8.3 \\
			HAT       & 52M  & 0.4 & 1.3 & 3.0 & 4.9 & 6.9 & 11.2 & 15.0 & 16.7 & 7.4 \\
			U-Net       & 68M  & 0.3 & 1.3 & 3.3 & 5.5 & 7.8 & 12.7 & 17.0 & 19.1 & 8.4\\
			U-Net (residual)   & 68M  & 0.2 & 1.3 & 3.2 & 5.3 & 7.5 & 12.4 & 16.5 & 18.5 & 8.1\\
			
			\midrule
			DM-S U-Net & 69M  & 0.6  & 1.4 & 3.1 & 4.9 & 6.9 & 11.0 & 14.6 &  16.2 & 7.3 \\ 
			DM-M U-Net   & 225M  & 0.2&  1.0 & 2.5 & 4.1 & 6.0 & 10.1 & 13.8 & 15.6 & 6.7\\ 
			DM-M U-Vit   & 210M  & 4.3& 3.2 & 4.1 & 5.6 & 7.4 & 11.3 & 14.6 & 16.2 & 8.3 \\ 
			
			\midrule
			FLEX-SR‑S    & 50M & \underline{\textit{0.2}} & \textit{1.0} & \textit{2.3} & \textit{3.9} & 5.5 & \textit{9.2} & \underline{\textit{12.5}} & \underline{\textit{14.1}} & 6.1 \\ 
			FLEX-SR‑M    & 200M & \underline{\textit{0.2}} & \underline{\textit{0.9}} & \underline{\textit{2.1}} & \underline{\textit{3.6}} & 5.2 & \underline{\textit{9.0}} & \underline{\textit{12.5}} & \textit{14.2} & \underline{\textit{6.0}} \\ 
			\midrule
			FLEX-MT‑S         & 58M  & \underline{\textit{0.2}} & 1.1 & 2.4 & 3.9 & 5.6 & 9.2 & 12.6 & 14.2 & 6.2 \\ 
			
			FLEX-MT‑M         & 240M & \underline{\textit{0.2}} & \underline{\textit{0.9}} & 2.2 & \underline{\textit{3.6}} & \underline{\textit{5.2}} & \textbf{8.9} & \textbf{12.3} & \textbf{14.0} & \textbf{5.9} \\   
			
			FLEX-MT‑L        & 1000M   & \textbf{0.1} & \textbf{0.8} & \textbf{2.0} & \textbf{3.5} & \textbf{5.1} & \underline{\textit{9.0}} & 12.6 & 14.4 & \textbf{5.9} \\ 
			\bottomrule
	\end{tabular}}
\end{table}

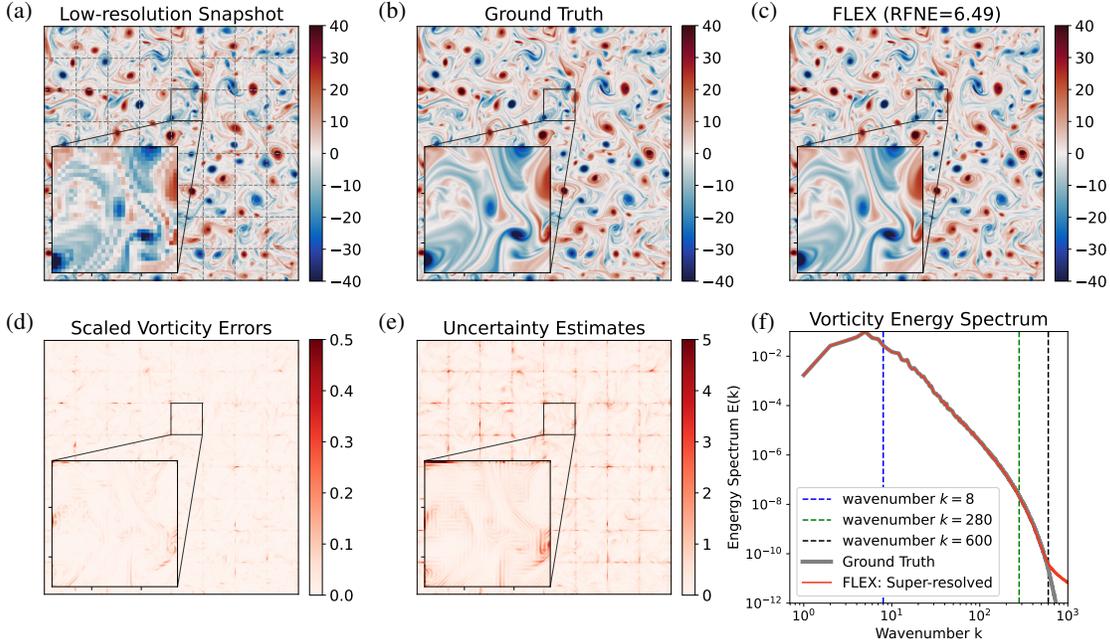
\begin{figure}[!t]
	\centering
	\begin{overpic}[width=0.94\linewidth] 
		{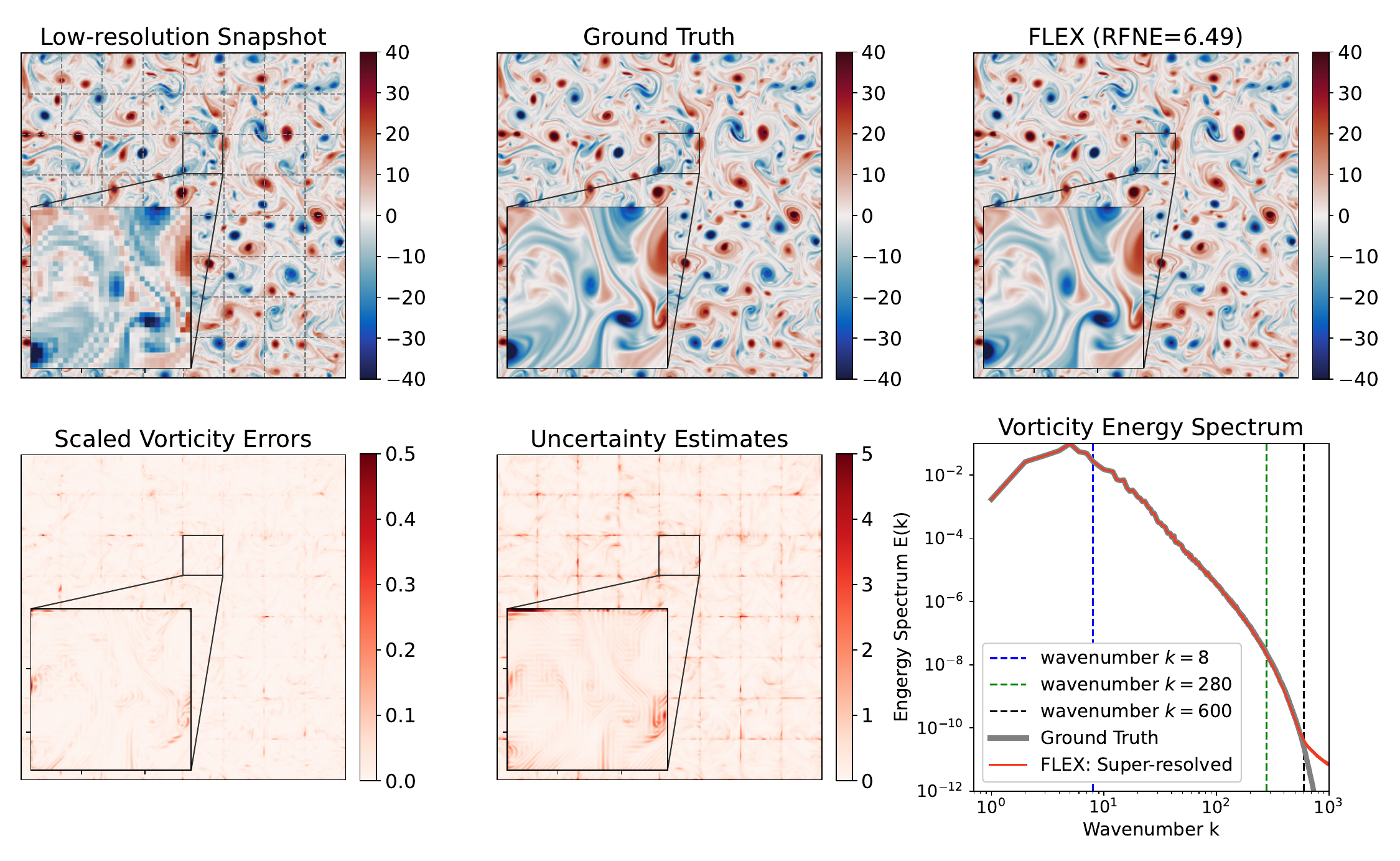}
		\put(-2, 58.5) {\small (a)}
		\put(32, 58.5) {\small (b)}
		\put(66, 58.5) {\small (c)}
		\put(-2, 30) {\small (d)}
		\put(32, 30) {\small (e)}
		\put(66, 30) {\small (f)}
	\end{overpic}
	\vspace{-0.3cm}
	\caption{Example snapshot demonstrating FLEX’s performance on vorticity field super-resolution at $Re = 16{,}000$. (a) Low-resolution vorticity snapshot with patch boundaries. (b–c) Comparison between the ground truth (b) and FLEX’s super-resolved output (c). (d) Error map showing small boundary artifacts. (e) Spatial uncertainty map estimated from ensembles showing higher uncertainty near complex vortex interactions. (f) Vorticity spectrum comparing reconstruction to the ground truth.}
\label{fig:SR}
\end{figure}

Figure~\ref{fig:SR} illustrates qualitative results for FLEX on the super-resolution task at $Re=16{,}000$. Despite operating on patches, the model produces coherent global reconstructions with minimal boundary artifacts. It accurately recovers both fine-scale eddies and large-scale flow structures. As shown in Figure~\ref{fig:SR}(f), FLEX also closely matches the ground truth in terms of the vorticity power spectrum, capturing energy across both the inertial and dissipative ranges up to wavenumbers $k \sim 600$.

We can generate multiple samples to estimate the predictive uncertainty. As shown in Figure~\ref{fig:SR}(e), the spatial standard deviation map highlights regions of uncertainty that reflect complex vortex interactions. A pull distribution analysis (Figure~\ref{fig:SR_UQ}) indicates that the uncertainty estimates are well calibrated, although slightly overconfident ($\sigma \approx 1.78$ instead of $1$). We see that $\sigma$ decreases as a function of the number of diffusion steps and the ensemble samples. 

\begin{figure}[!t]
\centering
\begin{overpic}[width=0.94\linewidth]
	{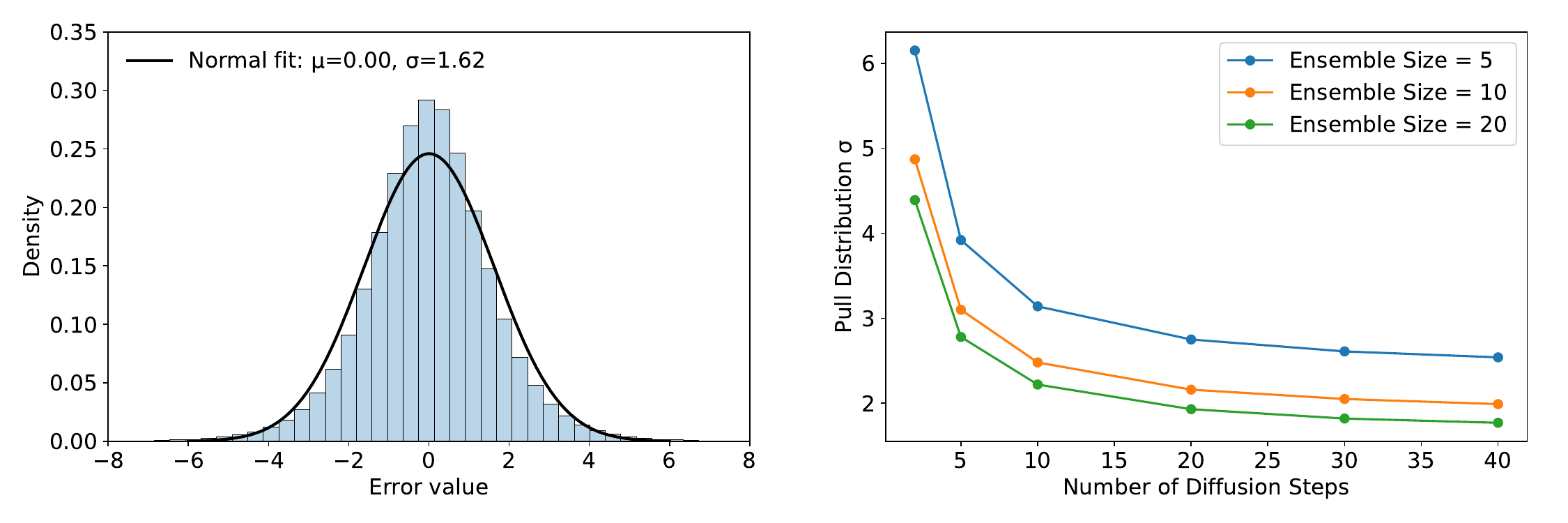}
	\put(-0.5, 31) {\small (a)}
	\put(51, 31) {\small (b)}
\end{overpic}\vspace{-0.3cm}
\caption{ (a) Pull distribution analysis shows that the model provides unbiased uncertainty estimates, while being slight overconfident. (b) We show that the overconfidence decreases as a function of ensemble size and number of diffusion steps, yet does not reach $\sigma=1$.}
\label{fig:SR_UQ}
\end{figure}

\begin{figure}[!t]
\centering
\begin{overpic}[width=0.94\linewidth]
	{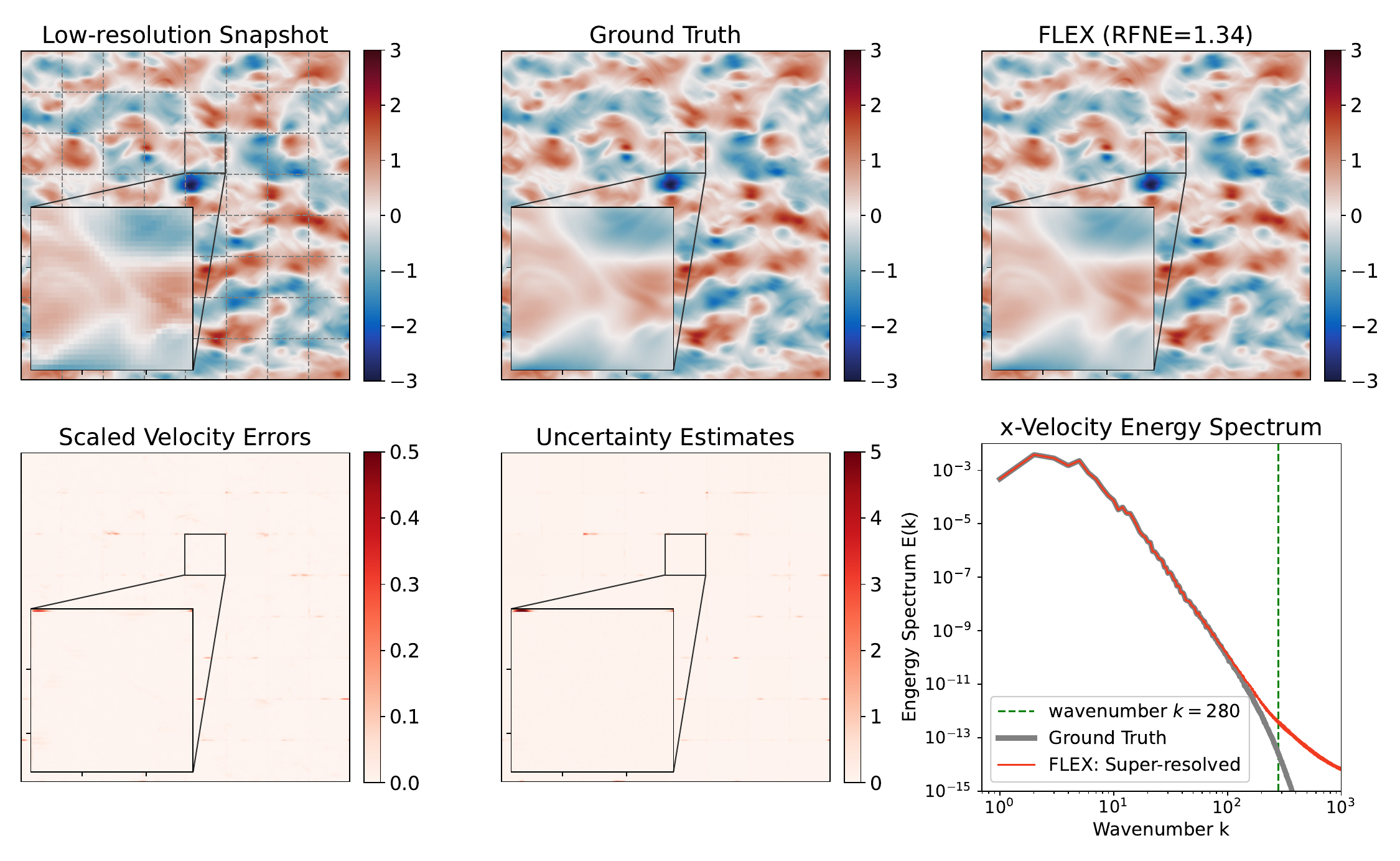}
	\put(-2, 58.5) {\small (a)}
	\put(32, 58.5) {\small (b)}
	\put(66, 58.5) {\small (c)}
	\put(-2, 30) {\small (d)}
	\put(32, 30) {\small (e)}
	\put(66, 30) {\small (f)}
\end{overpic}
\vspace{-0.3cm}
\caption{FLEX’s zero-shot performance on velocity field super-resolution at $Re = 16{,}000$.}
\label{fig:velocity_zero_shot}
\end{figure}

\begin{figure}[!t]
\centering
\begin{overpic}[width=0.94\linewidth]
	{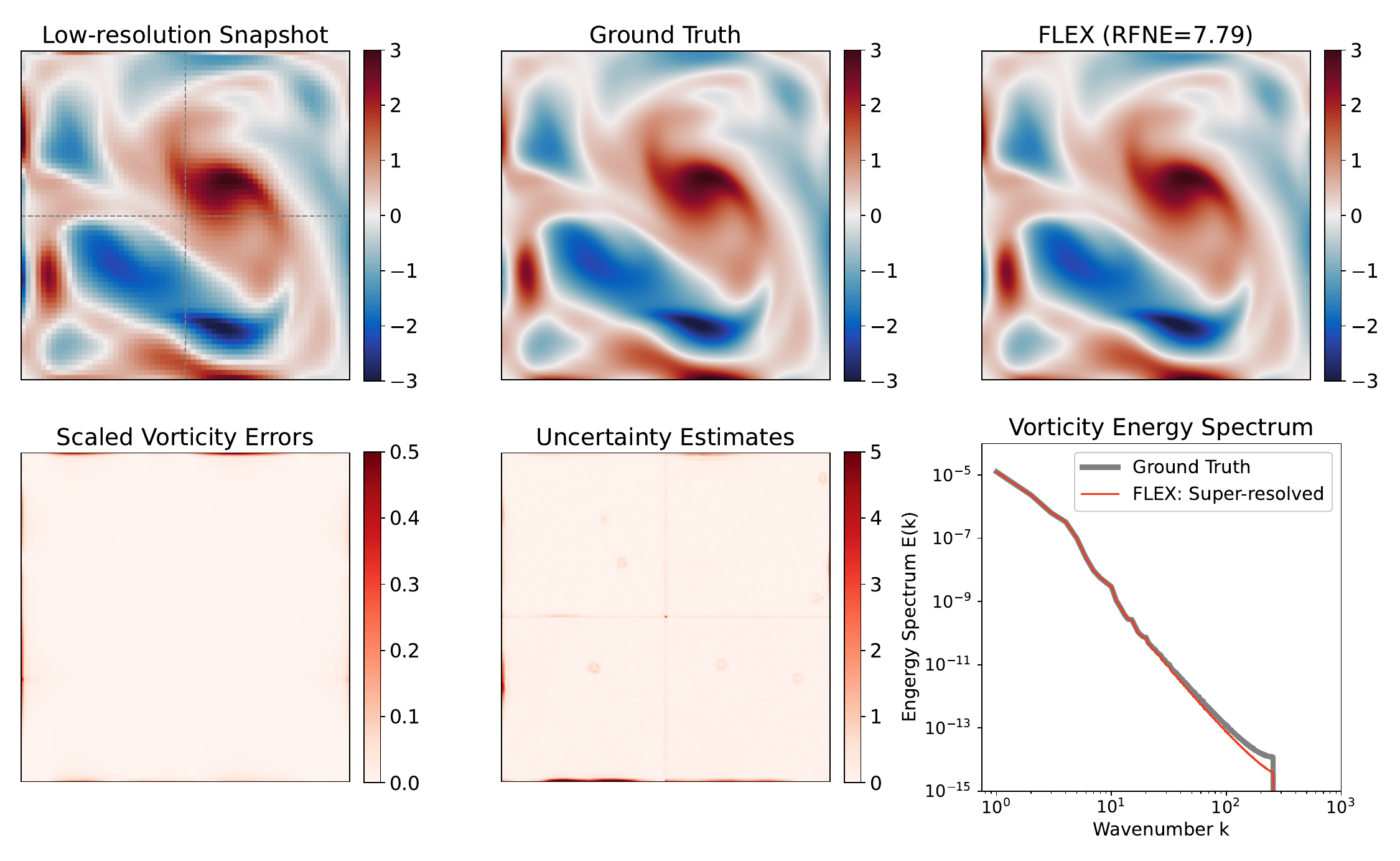}
	\put(-2, 58.5) {\small (a)}
	\put(32, 58.5) {\small (b)}
	\put(66, 58.5) {\small (c)}
	\put(-2, 30) {\small (d)}
	\put(32, 30) {\small (e)}
	\put(66, 30) {\small (f)}
\end{overpic}
\vspace{-0.3cm}
\caption{FLEX’s zero-shot performance on inhomogeneous NS data from PDEbench with Dirichlet boundary conditions. Note, that this is a new boundary condition not seen during training.}
\label{fig:pdebench_zero_shot}
\end{figure}

To test FLEX's ability to generalize to unseen physical observables, we demonstrate its performance on a velocity field, despite being trained only on vorticity. Figure~\ref{fig:velocity_zero_shot} shows a zero-shot super-resolved velocity field at $Re = 16{,}000$. We further evaluate FLEX on inhomogeneous Navier–Stokes data from PDEbench~\cite{takamoto2022pdebench}, which uses a Dirichlet boundary condition (BC) that is different from the periodic BC seen during training. Figure~\ref{fig:pdebench_zero_shot} shows a super-resolved vorticity field. 
This suggests that the learned physical representations can be generalized to related observables without supervision.

\subsection{Forecasting Future States}
\label{sec:results_forecasting}

Next, we evaluate FLEX on forecasting future vorticity fields, where the model is conditioned on the two most recent frames and predicts the next state. Predictions are recursively fed back to forecast over $s = 50$ time steps. Forecast accuracy is measured using the Pearson correlation coefficient between predictions and ground truth at each step, with values above $0.95$ considered highly accurate.

Figure~\ref{fig:FC}(a) shows the results for moderate 2D turbulence ($Re = 1{,}000$), where FLEX-FC-M maintains correlations above $0.95$ for approximately 40 steps. In the higher 2D turbulence regime ($Re = 12{,}000$), shown in Figure~\ref{fig:FC}(b), this horizon is reduced to around 30 steps. The multitask model, FLEX-ML-M, outperforms the single-task variant in both regimes. These results also show the advantage of the FLEX backbone over a diffusion model with a U-Net with attention~\cite{nichol2021improved}. The ViT backbone fails to produce reasonable forecasts (details omitted). FLEX also outperforms strong baselines, including a modified SwinIR model adapted for forecasting, FourCastNet~\cite{pathak2022fourcastnet} using Fourier token mixing with a ViT backbone, and the recently proposed DYffusion~\cite{ruhling2023dyffusion} framework using a U-Net with attention. Accuracy degrades beyond 50 steps due to compounding errors.

We also assess FLEX on generalization to out-of-distribution settings. Figure~\ref{fig:FC}(c) shows zero-shot predictions of velocity fields, a variable not seen during training, where FLEX maintains high correlation, demonstrating strong generalization between variables. In Figure~\ref{fig:FC}(d), we evaluate FLEX on PDEBench’s inhomogeneous Navier–Stokes data with Dirichlet boundary conditions. Surprisingly, performance in this domain-shifted setting exceeds that of the in-distribution test cases, suggesting that lower-fidelity simulations may present fewer challenges despite structural differences. Other baselines fail to generalize in this scenario, likely because they cannot adjust to the substantially lower residual magnitudes, leading to misaligned feature representations.

Figures~\ref{fig:forecast_vorticity}, \ref{fig:forecast_velocity}, and \ref{fig:forecast_pdebench} in App. \ref{app:results} show example vorticity and velocity forecasts at $Re = 12{,}000$, as well as a vorticity snapshot from PDEBench. In general, the predictions remain visually coherent and physically plausible even after 30–40 steps.

\begin{figure}[!t]
\centering
\begin{overpic}[width=0.48\linewidth]
	{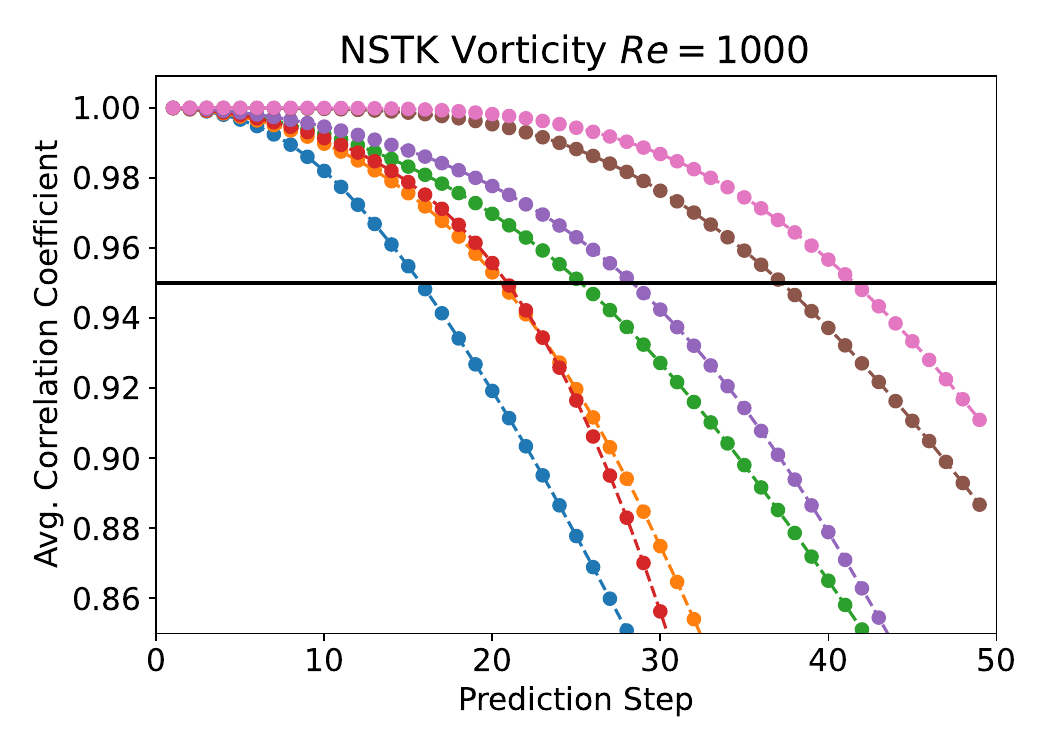}
	\put(1, 64) {\small (a)}
\end{overpic}
~
\begin{overpic}[width=0.48\linewidth]
	{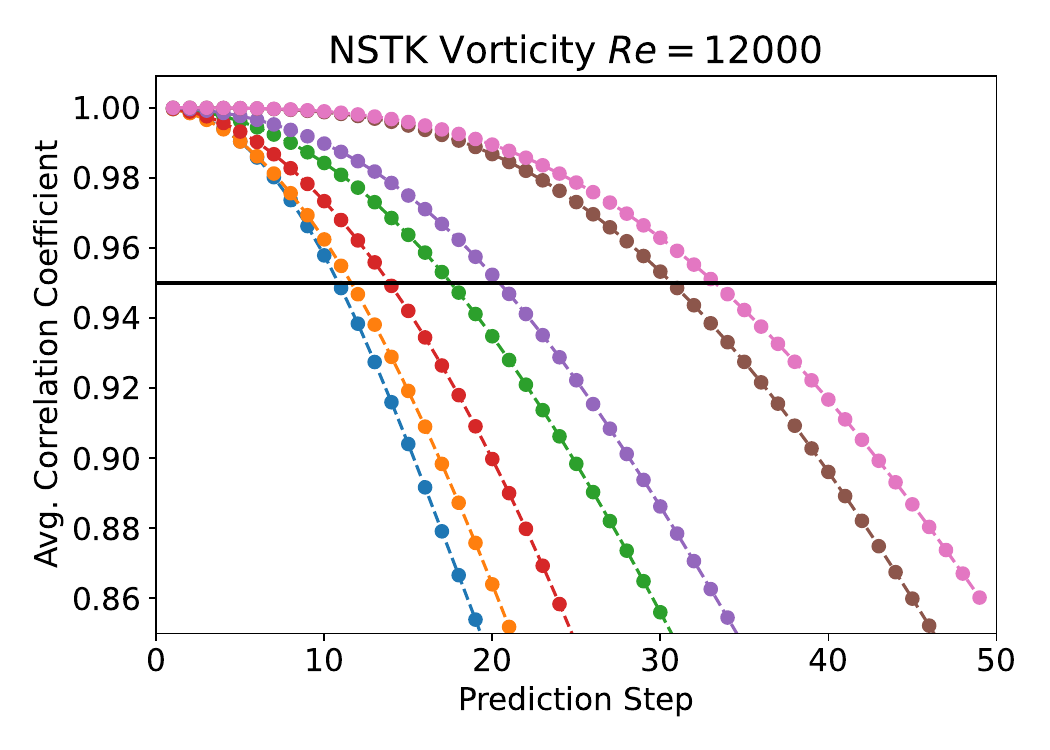}
	\put(1, 64) {\small (b)}
\end{overpic}

\begin{overpic}[width=0.48\linewidth]
	{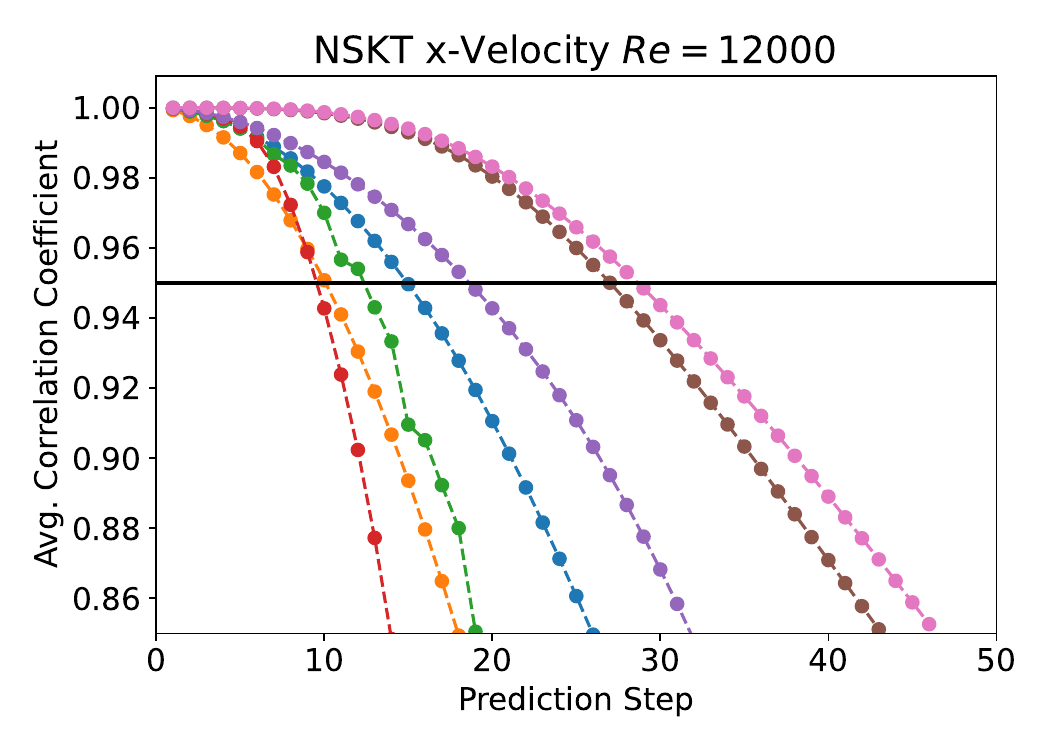}
	\put(1, 64) {\small (c)}
\end{overpic}
~
\begin{overpic}[width=0.48\linewidth]
	{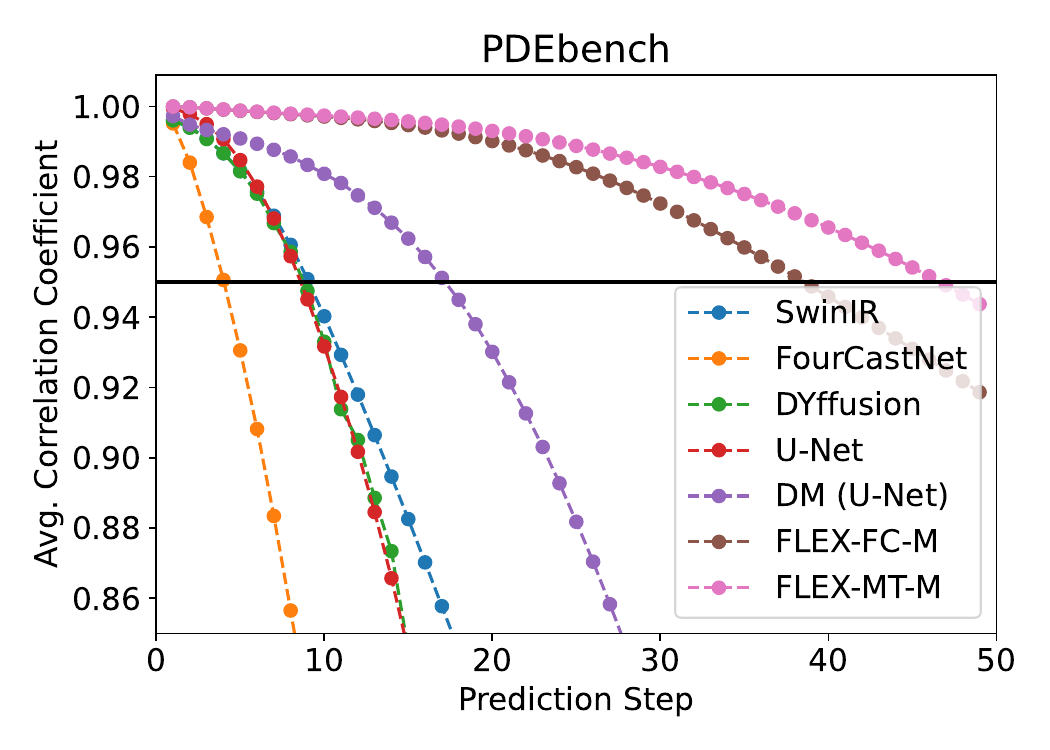}
	\put(1, 64) {\small (d)}
\end{overpic}
\vspace{-0.3cm}
\caption{Forecasting accuracy over 50 time steps. (a–b) Average Pearson correlation between predicted and ground-truth vorticity fields for moderate ($Re = 1{,}000$) and high ($Re = 12{,}000$) 2D turbulence. FLEX outperforms baselines including SwinIR, FourCastNet, and DYffusion.
	Moreover, we show zero-shot generalization to (c) velocity fields, and (d) PDEBench with Dirichlet boundaries.} 
	\label{fig:FC}
\end{figure}

\subsection{Ablation Studies}
\label{sec:ablations}


\begin{wraptable}{r}{0.55\textwidth}
\centering
\vspace{-1.0cm}
\caption{Ablation study on model components.}
\label{tab:ablation}
\begin{adjustbox}{max width=0.53\textwidth}
\begin{tabular}{cccc | cc c | ccc}
	\toprule
	Task & FLEX & U‑Net & ViT & \shortstack{Weak\\Cond.} & \shortstack{Strong\\Cond.} & Ensembles & 24k & 32k & 36k \\
	\midrule    
	SR & \xmark & \cmark & \xmark & \xmark & \cmark & \xmark & 10.5 & 14.1 & 15.8 \\
	
	SR & \xmark & \cmark & \xmark & \cmark & \xmark & \xmark & 9.8 & 13.3 & 15.1 \\
	
	SR & \xmark & \xmark & \cmark & \cmark & \xmark & \xmark & 10.8 & 14.0 & 15.7  \\
	
	\midrule
	
	SR & \cmark & \xmark & \xmark & \cmark & \xmark & \xmark & 9.3 & 12.9 &  14.6\\

	MT & \cmark & \xmark & \xmark & \cmark & \xmark & \xmark & 9.2 & 12.7 & 14.4 \\

	MT & \cmark & \xmark & \xmark & \cmark & \xmark & \cmark & \textbf{8.9} & \textbf{12.3} & \textbf{14.0} \\

	\bottomrule
\end{tabular}
\end{adjustbox}
\end{wraptable}
Table~\ref{tab:ablation} presents an ablation study for super-resolution, comparing medium-sized models ($\sim$240M parameters) across key architectural and training components. FLEX with a pure ViT model performs the worst, while augmenting a U-Net with attention layers yields moderate gains. The hybrid approach, which combines convolutional encoders with a latent Transformer, achieves the best accuracy. We find that weak conditioning through shallow skip connections provides a small but consistent performance boost. In contrast, removing the task-specific encoder and relying solely on strong encoder conditioning through channel concatenation degrades performance. Note that all models here use strong decoder conditioning. Finally, diffusion-based ensembling (using 10 samples) improves the results across all Reynolds numbers. 

\section{Conclusion}

We introduced FLEX, a new backbone architecture for generative modeling of spatio-temporal physical systems using diffusion models. FLEX addresses key limitations of standalone U-Net and ViT backbones by combining them into a hybrid architecture that enables (i) global dependency modeling in latent space, (ii) scalable and memory-efficient training, and (iii) flexible conditioning through task-specific encoders. In addition to architectural innovations, FLEX operates in residual space and adopts a velocity-based parameterization, both of which contribute to improved training stability and predictive performance.
Our results demonstrate that FLEX consistently outperforms U-Net and ViT-based diffusion models on both super-resolution and forecasting tasks involving high-resolution 2D turbulence. In particular, the multitask FLEX-MT model outperforms in a range of settings, including generalization to out-of-distribution physical variables and boundary conditions.

\textbf{Limitations.} While the considered forecasting task has limited direct practical domain utility, we view it as a valuable benchmark for stress testing spatio-temporal generative models. In contrast to low-pass filtered or low-fidelity simulations or very low Reynold numbers that enable longer forecast horizons, our setup challenges models and shows their current limitations. However, FLEX extends the forecast horizon compared to baselines, which can benefit applications such as reduced-cost simulation or real-time control where approximate forecasts are sufficient.

We did not evaluate all possible backbone variants due to the high computational cost of training large-scale diffusion models. We expect that other pure ViT backbones would exhibit similar shortcomings to the one evaluated here. Furthermore, due to the size of the models, we did not compute results over multiple random seeds, and therefore we do not report standard deviations.

\begin{ack}
This work was supported by the Laboratory Directed Research and Development Program of Lawrence Berkeley National Laboratory under U.S. Department of Energy Contract No. DE-AC02-05CH11231. We would also like to acknowledge the U.S. Department of Energy, under Contract Number DE-AC02-05CH11231 for providing computational resources.
\end{ack}

{
\small
\bibliographystyle{plain}
\bibliography{references}
}


\clearpage
\appendix

\section*{Appendix}

\section{Related Work} \label{app_relatedwork}

Diffusion models were first introduced as a thermodynamically inspired generative model~\cite{sohl2015deep}, which presented the idea of gradually corrupting data with noise and then reversing the process to recover the original sample. This concept was later revisited through the lens of scalable score-based modeling~\cite{song2019sliced} and made practical through the simple framework of denoising diffusion probabilistic models (DDPMs)~\cite{ho2020denoising}. The strong performance of DDPMs on image generation tasks sparked a great interest in diffusion models. Since then, research has evolved mainly along two lines: (1) understanding and improving the \textit{diffusion process}, including the development of new variants; and (2) designing better \textit{backbone networks}.

\subsection{Methods for Diffusion Models}

Early work focused on noise schedules and training objectives, including improved DDPMs~\cite{nichol2021improved}, continuous-time formulations based on score-based stochastic differential equations~\cite{song2021score}, and accelerated samplers such as DDIMs~\cite{song2021denoising} and DPM-Solver~\cite{lu2022dpm}. More recently, distillation-based approaches, such as consistency models~\cite{song2023consistency} and latent space models such as  EDM~\cite{karras2024edm2} and VDM~\cite{kingma2021variational} have further improved sampling speed and stability. In parallel, alternative formulations, such as flow matching~\cite{lipman2022flow} propose new ways to define the generative process beyond traditional score-based objectives. In particular, \cite{lim2024elucidating} explored the potential of the flow matching framework for probabilistic forecasting of spatio-temporal dynamics.

\subsection{Backbones for Diffusion Models}

Much effort has been put into replacing or augmenting the U-Net~\cite{ronneberger2015u} backbone, used as a denoising network in DDPMs~\cite{ho2020denoising}. Although U-Nets provide strong local inductive biases, recent work has shown that backbone design plays a important role in sample quality and scalability. Nichol et.  al.~\cite{nichol2021improved} proposed augmenting U-Nets with multi-head self-attention at coarse resolutions to capture global context.
More recent approaches explore fully Transformer-based alternatives. These include Diffusion Transformers (DiT)~\cite{peebles2023scalable}, and flexible ViT variants such as FiT~\cite{lu2024fit} and FiTv2~\cite{wang2024fitv2}. Hybrid architectures that combine convolutional ResNet and Transformer blocks improve the trade-off between performance and memory~\cite{hoogeboom2023simple}, while SiD2~\cite{hoogeboom2024simpler} shows that simplified pixel-space models with minimal skip connections can still achieve competitive performance for full image generation, while significantly reducing memory usage.
Models such as Dimba~\cite{fei2024dimba}, U-Shape Mamba~\cite{ergasti2025u} and Diffusion State-Space Models (DiS)~\cite{fei2024scalable} leverage state space models as alternative for Transformers.

\subsection{Applications}

Conditional diffusion models have become widely adopted across a range of applications, including super-resolution, video generation, and scientific modeling. By conditioning on auxiliary inputs such as low-resolution data, past frames, or physical parameters, these models achieve strong performance in both computer vision and scientific domains.

\textbf{Diffusion Models for Super-Resolution.} The SR3 model~\cite{saharia2022image} introduced diffusion models for image super-resolution, achieving GAN-level perceptual quality by conditioning on bicubic-upsampled inputs. Subsequent work has expanded on this framework in various ways. SRDiff~\cite{li2022srdiff} predicts residuals instead of full images, SinSR~\cite{wang2024sinsr} explores single-step generation, and ResDiff~\cite{shang2024resdiff} leverages convolutional neural network (CNN) to recover domaint low-frequency features.

\textbf{Diffusion Models for Video Generation.} 
Video Diffusion Models (VDMs)~\cite{ho2022video} extend DDPMs to 3D spatiotemporal tensors, allowing generative modeling of short video clips. To reduce computational costs, latent space pipelines such as Align-Your-Latents~\cite{blattmann2023align}, and MagicVideo~\cite{zhou2022magicvideo} attempt to decouple motion and appearance. Transformer-based models such as Latte~\cite{ma2024latte} and CogVideoX~\cite{yang2024cogvideox} have further improved video generation.

\textbf{Diffusion Models for Spatio-temporal Predictions.} Diffusion models are also increasingly used in scientific applications, particularly for super-resolution and forecasting in high-dimensional spatio-temporal systems. Diffusion-based super-resolution has been applied to numerical weather prediction~\cite{watt2024generative, wan2024statistical}. 
In weather forecasting, FourCastNet-Diffusion~\cite{pathak2024kilometer} scales diffusion surrogates to kilometer-resolution grids, while GenCast~\cite{price2023gencast} formulates forecasting as an ensemble diffusion task, outperforming operational baselines on a wide range of meteorological targets. 
StormCast~\cite{pathak2024kilometer} is diffusion models for predicting the evolution of thunderstorms and mesoscale convective systems. Their implementation combines a deterministic regression approach with a diffusion model to learn the conditional residual distribution using elucidated diffusion models (EDM)~\cite{karras2022}. 
In fluid dynamics, models such as Turbulent Diffusion~\cite{kohl2023turbulent} and DiffFluid~\cite{luo2024difffluid} predict short-term flow evolution from initial conditions. 
Oommen et al.~\cite{oommen2024integrating} propose to combine neural operators (NO) with diffusion models to model turbulent fluid flows. Here, the NO component is focusing on learning low-frequency features, while the diffusion model is used to reconstruct the high-frequency features. 
GenCFD extends diffusion models for the challenging task of modeling 3D turbulence~\cite{molinaro2024generative}.

FLEX is compatible with existing pipelines and can serve as a drop-in backbone for both super-resolution and forecasting tasks. We note that much of the current work on spatiotemporal diffusion modeling does not yet incorporate recent architectural advances; see Table~\ref{tab:related_work} for a summary.

\begin{table}[!t]
\centering
\caption{Comparison with related diffusion-based models.}
\label{tab:related_work}
\scalebox{0.85}{
\begin{tabular}{lccccccr}
\toprule

\textbf{Model} & \textbf{Task(s)} & \textbf{Domain} & \textbf{Sampler} & \textbf{EMA} & \textbf{Velocity} & \textbf{Backbone} & \textbf{Residual} \\
\midrule
DM Downscaling~\cite{watt2024generative} & SR & Climate & DDIM & \xmark & \xmark & U-Net & \xmark \\
DM Downscaling~\cite{wan2024statistical} & SR & Climate & - & - & \xmark & U-ViT & \xmark \\

GenCast~\cite{price2023gencast} & FC & Climate & - & - & \xmark& U-Net & \xmark \\
StormCast~\cite{pathak2024kilometer} & FC & Climate & EDM & \cmark & \xmark& U-Net & \cmark \\
Turbulent Diffusion~\cite{kohl2023turbulent} & FC & Fluids & DDPM & \xmark & \xmark & U-Net & \xmark \\
DiffFluid~\cite{luo2024difffluid} & FC & Fluids & DDPM & \xmark & \xmark & U-Net & \xmark \\

NO-DM~\cite{oommen2024integrating} & FC & Fluids & DPM & \xmark & \xmark & NO+ U-Net & \xmark \\

DiffCFD~\cite{molinaro2024generative} & FC & Fluids & - & -- & \xmark & U-ViT & \xmark \\

FLEX (ours) & SR \& FC & Fluids & DDIM & \cmark & \cmark & FLEX & \cmark \\

\bottomrule
\end{tabular}}%
\end{table}

\section{Theoretical Results} \label{app_theory_velocitymodel}

We first provide minimal background on the diffusion model and the velocity parametrization that we consider in Sec. \ref{sec_diffusion_residual}. After recalling our setting, we derive the optimal velocity model and study its properties, which we use to motivate our choice of training the diffusion model using residual samples (rather than the raw data samples).

\subsection{Background and Setting} \label{app_setting}

Let $Z_0  \sim p_0$, where $Z_0 \in \mathbb{R}^d$ is a random variable representing data  samples drawn from $p_0$.
The diffusion model consists of forward-backward process of noising and denoising. The forward noising process is pre-specified and describes gradual degradation of samples from $Z_0$ towards Gaussian noise over time via \begin{equation}
Z_t = \alpha(t) Z_0 + \sigma(t) \epsilon, \quad \epsilon \sim \mathcal{N}(0, I), \quad t \in [0,1],
\label{eq:forward_diffusion}
\end{equation}
where $\alpha(t)$ and $\sigma(t)$ are  time-differentiable scalar-valued functions describing the noise schedule. The $Z_t$ are realizations of  the probability density path (transition kernel) $p_t(Z|Z_0) := p_{t|0}(Z_t | Z_0) =  \mathcal{N}(\alpha(t) Z_0, \sigma^2(t) I)$ for $t \in [0,1]$. The marginal distribution $p_t(Z) = \mathrm{Law}(Z_t)$ is thus the Gaussian smoothing of $p_0$ under the scaling $Z = \alpha(t) Z_0 + \sigma(t) \epsilon$, $\epsilon \sim \mathcal{N}(0,I)$:
\begin{equation}  
    p_t(Z) = \int p_t(Z|Z_0) p_0(Z_0) dZ_0, \label{eq_gaussiansmoothing}
\end{equation} 
which can be viewed as an interpolation between $p_0$ and $p_1$ \cite{albergo2023stochastic}. 

It has been shown that this noising procedure can be equivalently (in the sense of marginal law) described via the stochastic differential equation (SDE) \cite{song2021score}:
\begin{equation}
dZ_t = f(t) Z_t \, dt + g(t) \, dW_t, 
\quad t \in [0, 1],
\end{equation}
where $W_t$ denotes a standard Wiener process,  and $f(t)$ and $g(t)$ can be related to $\alpha(t)$ and $\sigma(t)$ via: $f(t)=\frac{d \log \alpha(t)}{dt} = \frac{\dot{\alpha}(t)}{\alpha(t)}$ and $g^2(t) = \frac{d \sigma^2(t)}{dt} - 2 \frac{d \log \alpha(t)}{dt} \sigma^2(t) = 2 \sigma^2(t) \left( \frac{\dot{\sigma}(t)}{\sigma(t)} - \frac{\dot{\alpha}(t)}{\alpha(t)} \right)$. Denoting its marginal distribution as $q_t = \mathrm{Law}(\tilde{Z}_t)$, the equivalence is given by $q_t = p_{1-t}$ for $t \in [0,1]$.

Generating new samples from noise requires reversing the forward process. The reverse denoising process can be shown to be described by the reverse-time SDE~\cite{cao2023exploring}:
\begin{equation} \label{eq_general_reverseSDE}
d\tilde{Z}_t = \left( -f(1-t) \tilde{Z}_t + \frac{g^2(1-t) + h^2(1-t) }{2} \nabla_{\tilde{Z}_t} \log p_{1-t}(\tilde{Z}_t | C) \right) dt + h(1-t) \, d\bar{W}_t,
\end{equation}
where $\tilde{Z}_0 \sim p_1$, $h(t)$ is an arbitrary real-value function, $\bar{W}_t$ is a Wiener process running backward in time, $C$ denotes conditioning variables, and $\nabla_{\tilde{Z}_t} \log p_{1-t}(\tilde{Z}_t | C)$ is the ground truth score function (conditioned on $C$).  Setting $h = g$ leads to the backward SDE of \cite{song2020score}. Meanwhile,
in the Probability Flow ODE formulation~\cite{song2020score}, setting the  diffusion coefficient to zero (i.e., setting $h=0$) leads to a deterministic trajectory. The target distribution $p_0$ can then be sampled by first sampling $\tilde{Z}_0 \sim q_0 = p_1$ and then evolving $\tilde{Z}_t$ according to (\ref{eq_general_reverseSDE}) to obtain new samples $\tilde{Z}_1 \sim q_1 = p_0$.

\noindent {\bf Setting.} We initialize the diffusion model with residual samples, which we denote as $R$, instead of the full/raw data samples, denoted as $Z^{raw}$. We also use the variance-preserving formulation that satisfies $\alpha^2(t) + \sigma^2(t) = 1$ for all $t \in [0,1]$. In particular, we have chosen $\alpha(t) = \cos(\pi t/2)$, $\sigma(t) = \sin(\pi t/2)$ (see App. \ref{sec:model_details}), resulting in $f(t) = -\frac{\pi}{2} \frac{\sigma(t)}{\alpha(t)}$ and $g^2(t) = \pi \frac{\sigma(t)}{\alpha(t)}$. Moreover, instead of learning the score function directly, we choose to learn a velocity field, adopting the velocity parametrization given in Eq. (\ref{eq_velocity}), i.e., $v(t,R) = \alpha(t) \epsilon - \sigma(t) R$, which coincides with $\dot{Z}_t$ (up to a multiplicative factor): 
$\dot{Z}_t = \dot{\alpha}(t) Z_0 + \dot{\sigma}(t) \epsilon = \frac{\pi}{2} v(t, R)$. We observe that $v(t, R) = \dot{Z}_t/\sqrt{\dot{\alpha}^2(t) + \dot{\sigma}^2(t)}$, which is exactly the normalized velocity parametrization considered in \cite{zheng2023improved} (see Section 4.3 there). We are using a neural network to learn this velocity field.

\subsection{Optimal Velocity Model and Its Properties}

In the following, we work in the setting described in Subsection \ref{app_setting}, keeping the choice on the type of data $Z_0$ (or equivalently $p_0$) flexible, i.e., $Z_0$ can be either $R$ or $Z^{raw}$. To see how our velocity parametrization relates to the score function, we look at the optimal velocity field, which is the target velocity field that we aim to learn, given unlimited neural network capacity.
Omitting the conditioning variable $C$ for notational simplicity, the optimal velocity field can be derived as follows:
\begin{align}
v^*(t, Z) &= \arg\min_{v_\theta} \mathbb{E}[\|v_\theta(t, Z) - (\alpha(t) \epsilon - \sigma(t) Z_0)\|_2^2 \ | \ Z_t = Z] \nonumber \\
&= \mathbb{E}[\alpha(t) \epsilon - \sigma(t) Z_0 \ | \ Z_t = Z] \nonumber \\
&= \mathbb{E} \left[\alpha(t) \left(\frac{Z_t - \alpha(t) Z_0}{\sigma(t)} \right) - \sigma(t) Z_0 \ \bigg| \ Z_t = Z \right] \nonumber \\
&= \mathbb{E} \left[ \frac{\alpha(t)}{\sigma(t)} Z_t - \left( \frac{\alpha^2(t)}{\sigma(t)} + \sigma(t) \right) Z_0 \ \bigg| \ Z_t = Z \right] \nonumber \\
&= \frac{\alpha(t)}{\sigma(t)} Z - \frac{1}{\sigma(t)} \mathbb{E}[Z_0 \ | \ Z_t = Z] \nonumber \\
&=  \frac{\alpha(t)}{\sigma(t)} Z  - \frac{1}{\sigma(t)} \left(\frac{ Z + \sigma^2(t) \nabla_Z \log p_t(Z)}{\alpha(t)} \right) \nonumber \\
&= -\frac{\sigma(t)}{\alpha(t)} \left( Z + \nabla_Z \log p_t(Z) \right), \label{eq_v_intermsof_score}
\end{align}
where $\nabla_Z \log p_t(Z)$ is the ground truth score function. Note that we have used $\alpha^2(t) + \sigma^2(t) = 1$ to arrive at the forth and the last equality above, and we have applied  Tweedie's formula \cite{efron2011tweedie} in the sixth equality.

Rewriting $v^*(t, Z)$ in terms of $f$ and $g$ gives $v^*(t, Z) = \frac{\pi}{2}(f(t) Z - \frac{g^2(t)}{2} \nabla_Z \log p_t(Z))$, which is the drift of the Probability Flow ODE (up to a scalar multiplicative factor). Predicting the drift of the diffusion ODE was also considered in \cite{zheng2023improved}, which has  shown that it can alleviate the imbalance problem in noise prediction.

The target velocity field plays a central role in both learnability and training stability. For instance, one would expect target velocity fields that are smooth and Lipschitz continuous are easier to be learnt, whereas sharp gradients lead to unstable training behavior. Also, a velocity field with smaller variances under the sample distributions allows more stable training. Thus, the choice on the type of samples used to initialize the diffusion model is also important. With these in mind, we shall study the variance and gradient of the optimal velocity field.

The following result relates the expected squared norm of the optimal velocity with the distance between the marginal density path $p_t$ and  $\mathcal{N}(0,I)$. \\

\begin{prop} \label{prop_1} Let $D_F(p \| q)$ denote the Fisher divergence between two probability density functions $p$ and $q$. We have, for $t \in [0,1]$,
\begin{equation}
    \mathbb{E}_{p_t}[\|v^*(t, Z)\|^2] = \left(\frac{\sigma(t)}{\alpha(t)} \right)^2 D_F(p_t \| \mathcal{N}(0,1)).
\end{equation}
\end{prop}
\begin{proof}
Note that $v^*(t, Z)$ in (\ref{eq_v_intermsof_score}) can be expressed as a scaled difference between two score functions, i.e., $v^*(t, Z) = \frac{\sigma(t)}{\alpha(t)} \left( \nabla_Z \log r(Z)  -  \nabla_Z \log p_t(Z) \right)$,
where $r(Z) \sim \mathcal{N}(0, I)$ and $p_t(Z)$ is given in Eq. (\ref{eq_gaussiansmoothing}). The result then follows upon taking the squared norm and the expectation as well as using the definition of Fisher divergence.  
\end{proof}

Proposition \ref{prop_1} implies that, for a given $t$, the further away $p_t$ is from $\mathcal{N}(0,I)$, the  larger the variance of the velocity magnitude $\|v^*\|$. Therefore, starting from a distribution $p_0$ that is closer to $\mathcal{N}(0,I)$ in the diffusion model helps reducing this variance and could stabilize training.

\begin{figure}[!t]
    \centering
    \begin{overpic}[width=0.48\linewidth]
    {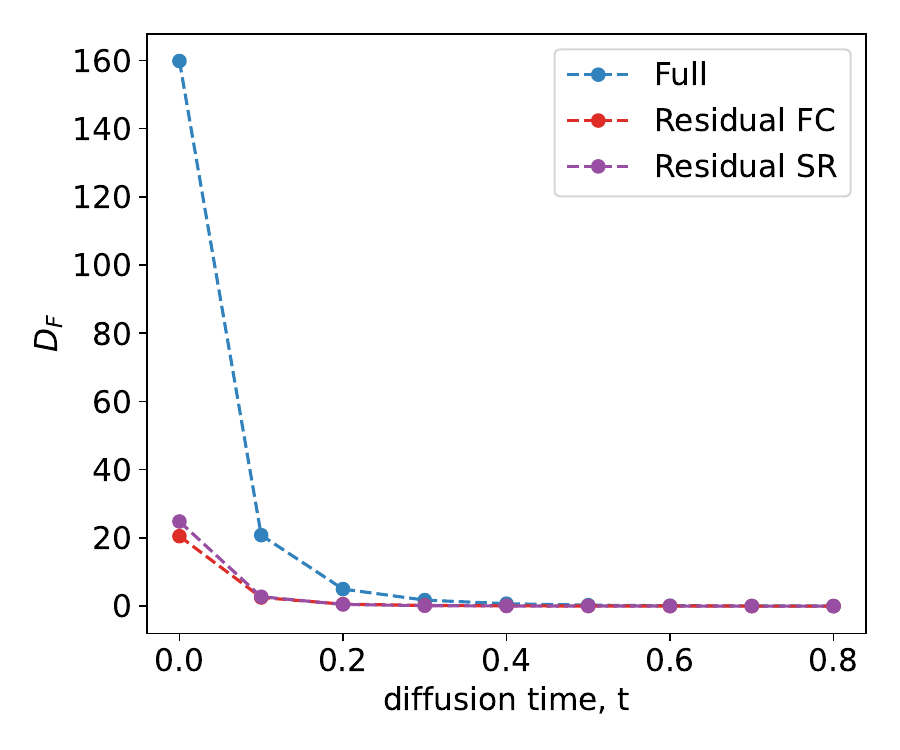}
    \end{overpic}
    ~
    \begin{overpic}[width=0.48\linewidth]
    {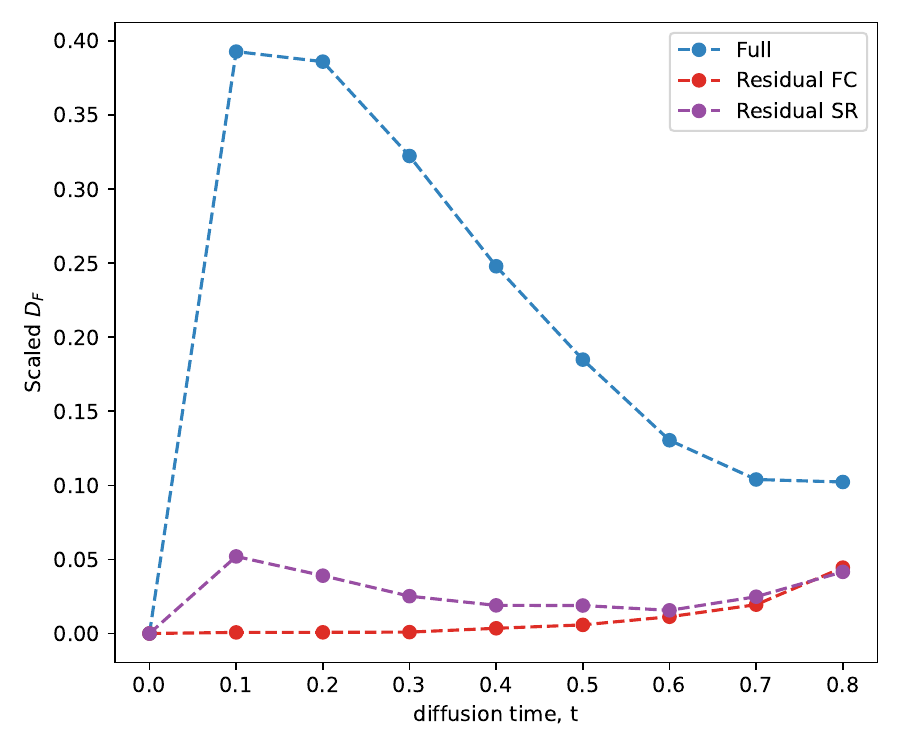}
    \end{overpic}

    \caption{On the left, we investigate how "Gaussian" the raw data samples and residual samples are at different diffusion time (noise levels). We used 40,000 small patches sampled from vorticity snapshots, and we gradually added noise (increasing diffusion time $t$ from 0 to 1). We then measured how close each patch distribution is to $\mathcal{N}(0,I)$ using a Fisher-divergence score (denoted as $D_F$). We show three cases: (1) the raw data, (2) the forecast residuals between the current and next patch, and (3) the super-resolution residuals between the true high-res patch and a simple upsampled low-res version. On the right, we plot estimates of $\mathbb{E}_{p_t}[\| v^*(t,Z)\|^2]$, which is a scaled $D_F$ according to Proposition \ref{prop_1}, for multiple diffusion times. }
    \label{fig:df}
\end{figure}

We shall use Proposition \ref{prop_1} and its implication to shed light on the benefits of using residual data $R$ instead of full data $Z^{raw}$ to initialize our velocity-parametrized diffusion model. We quantify how closely the residual-based and full-data based score function matches the standard Gaussian distribution by using the Fisher divergence
$$
D_F\bigl(p_t \| \mathcal{N}(0,1)\bigr)
= \int \bigl[\nabla_x \log p_t(x) + x \bigr]^2\,p_t(x)\,\mathrm{d}x.
$$

We draw \(N=40\,000\) random \(256\times256\) patches from the high-resolution \((2048\times2048)\) fluid flow simulations, \(\mathrm{Re}=16\,000\), and we add noise according to the cosine log-SNR diffusion schedule parameterized by \(t\in[0,1]\).  At each \(t\), we compute the following Monte Carlo estimate:
\[
D_F\bigl(p_t\|\mathcal{N}(0,1)\bigr)
\;\approx\;
\frac{1}{N}\sum_{i=1}^N
\Bigl[\nabla_x\log\hat p_t(x_i) + x_i\Bigr]^2,
\]
where \(\hat p_t\) is obtained via a Gaussian kernel density estimator over the sampled patches.
We compute and compare the estimates for three different types of noised data at each time \(t\):
\begin{itemize}
  \item \textbf{Original data:} $Z^{raw}_t = X_t$, the true patch vorticity.
  \item \textbf{Forecast residuals:} \(Z_t^{\rm temp}=X_t - X_{t-1}\), the temporal difference between snapshots.
  \item \textbf{Upsampled LR residuals:} \(Z_t^{\rm spatial}=X_t - \mathrm{Upsample}_8\bigl(X_t^{\rm LR}\bigr)\), the difference between the high-resolution patch and an 8× lower-resolution upsampled version using bicubic interpolation.
\end{itemize}

Figure~\ref{fig:df} shows the results. In the \emph{left panel}, we plot the estimate of
$D_F\bigl(p_t\|\mathcal{N}(0,1)\bigr)$ as a function of diffusion time $t$,
showing that all three divergences decrease as noise dominates, but that both residual-based scores lie closer to the Normal (smaller \(D_F\)) than the raw data case, particularly noticeable for small \(t\). In the \emph{right panel} we scale each curve with
$\tan^2\!\Bigl(\tfrac{\pi t}{2}\Bigr)$, 
defining
$\widetilde D_F(t)
=  \tan^2\!\Bigl(\tfrac{\pi t}{2}\Bigr) \cdot D_F\bigl(p_t\|\mathcal{N}(0,1)\bigr)$, which is precisely $\mathbb{E}_{p_t}[\|v^*(t, Z)\|^2]$, due to Proposition \ref{prop_1} and our choice of noise schedule. We see that the estimated $\mathbb{E}_{p_t}[\|v^*(t, Z)\|^2]$ is significantly lower when we use the residual samples (for both super-resolution and forecasting case) compared to using the raw data samples over a wide range of $t$. Therefore, using residual samples instead of raw data samples in the diffusion model helps reduce the variance of $\|v^*\|$ and could stabilize training. This motivates our choice of training the diffusion model in the residual space.

Although Proposition \ref{prop_1} together with Figure \ref{fig:df} are sufficient to justify the use of residual data instead of raw data to initialize the diffusion model, we provide  additional results on the gradient of the optimal velocity field, which could be of independent interest. 

Let $Cov_p(X,Y) := \mathbb{E}_p[(X-\mathbb{E}_p[X])(Y-\mathbb{E}_p[Y])^T]$ denote the covariance matrix of $X$ and $Y$, and let $p_{0|t}(Z_0 | Z_t)$ denote the distribution of $Z_0$ conditioned on the value of $Z_t$. Also, let $Cov_t(Z_0, Z_0) := Cov_{p_{0|t}(Z_0 | Z_t)}(Z_0,Z_0)$.  

We start with the following lemma on the Hessian of the score function. \\
\begin{lem} \label{app_lemma}
For the $p_t(Z)$ defined in Eq. (\ref{eq_gaussiansmoothing}), we have, for $t \in [0,1]$:
    \begin{equation}
        \nabla_Z^2 \log p_t(Z) = \frac{\alpha^2(t)}{\sigma^4(t)} Cov_t(Z_0,Z_0).
    \end{equation}
\end{lem}
\begin{proof}
    The result follows from a direct computation of the Hessian of the log density of $p_t$. For full details, see, e.g., Section A.1 in \cite{gottwald2025localized}. 
\end{proof}

The following proposition provides  bound on the gradient of the optimal velocity model in terms of the noise schedule and the maximum eigenvalue of a conditional covariance. \\

\begin{prop} \label{prop_2}
Let $r(Z) \sim \mathcal{N}(0,I)$. For $t \in (0,1)$,
    \begin{align}
        \| \nabla_Z v^*(t, Z)\|_{op} &\leq \frac{\sigma(t)}{\alpha(t)} \cdot \left\|\nabla_Z^2 \log p_t(Z) - \nabla_Z^2 \log r(Z)\right\|_{op}
    \end{align}
and
    \begin{align}
        \| \nabla_Z v^*(t, Z)\|_{op}  &\leq   \frac{\sigma(t)}{\alpha(t)}  +  \frac{\alpha(t)}{\sigma^3(t)} \cdot \lambda_{max}(Cov_t(Z_0, Z_0)),  
    \end{align}
where $\| A \|_{op} := \sup_{\|x\|_2 = 1} \|A x\|_2$ denotes the operator norm of the matrix $A$, and $\lambda_{max}(A_t)$ denotes the largest eigenvalue of the symmetric and positive semi-definite matrix $A_t$.
\end{prop}
\begin{proof}
    We compute, for $t \in (0,1)$,
    \begin{equation}
        \nabla_Z v^*(t, Z) = -\frac{\sigma(t)}{\alpha(t)}\left( \nabla_Z^2 \log p_t(Z) - \nabla_Z^2 \log r(Z))  \right).
    \end{equation}
    Thus, as $\frac{\sigma(t)}{\alpha(t)}  = \tan(\pi t/2) > 0$ on $(0,1)$,  
    \begin{align}
        \| \nabla_Z v^*(t, Z)\|_{op} &\leq \bigg|\frac{\sigma(t)}{\alpha(t)} \bigg| \cdot \left\|\nabla_Z^2 \log p_t(Z) - \nabla_Z^2 \log r(Z)\right\|_{op} \nonumber \\
        &= \frac{\sigma(t)}{\alpha(t)} \cdot \left\|\nabla_Z^2 \log p_t(Z) - \nabla_Z^2 \log r(Z)\right\|_{op},
    \end{align}
arriving at the first inequality in Proposition \ref{prop_2}. 

To obtain the second inequality, we start from
\begin{equation}
        \nabla_Z v^*(t, Z) = -\frac{\sigma(t)}{\alpha(t)}\left( \nabla_Z^2 \log p_t(Z) + I  \right).
\end{equation}
Using Lemma \ref{app_lemma}, we have:
\begin{equation}
        \nabla_Z v^*(t, Z) = - \frac{\sigma(t)}{\alpha(t)} I - \frac{\alpha(t)}{\sigma^3(t)} Cov_t(Z_0, Z_0). 
\end{equation}
Thus, as $\frac{\sigma(t)}{\alpha(t)} = \tan(\pi t/2), \frac{\alpha(t)}{\sigma^3(t)} = \frac{\cos(\pi t/2)}{\sin^3(\pi t/2)} > 0$ on $(0,1)$,
 \begin{equation}
        \| \nabla_Z v^*(t, Z)\|_{op} \leq 
        \frac{\sigma(t)}{\alpha(t)}  + \frac{\alpha(t)}{\sigma^3(t)} \|Cov_t(Z_0, Z_0)\|_{op} = \frac{\sigma(t)}{\alpha(t)}  + \frac{\alpha(t)}{\sigma^3(t)}  \cdot
        \lambda_{max}(Cov_t(Z_0, Z_0)).
    \end{equation}
\end{proof}

We observe that, with our choice of noise schedule, $\lim_{t \downarrow 0} \nabla_Z v^*(t, Z) = -\infty$ and $\lim_{t \uparrow 1} \nabla_Z v^*(t, Z) = - \infty$, which is why we restrict to the interval $(0,1)$ when deriving the bound in Proposition \ref{prop_2}. This unstable behavior near the boundaries is not suprising since the model attempts to sharply match a target distribution (at $t=1$) or reverse a singular noising process (at $t=0$). Proposition \ref{prop_2} suggests that using training data samples where  $\lambda_{max}(Cov_0(Z_0,Z_0))$ is small to initialize the diffusion model could help dampen the blow up of $\nabla_Z v^*$ near $t = 0$ and improve model stability. We estimate $\lambda_{max}(Cov_0(Z_0,Z_0))$ for the three types of data considered in Figure \ref{fig:df}, and we find that the estimated top eigenvalue for the residual  cases is at least one order of magnitude lower than for the raw data case. This strengthens the justification for training the diffusion model in the residual  space.

\section{Diffusion Model Details}
\label{sec:model_details} 

We now describe how FLEX instantiates the general framework described in Sec. \ref{sec_diffusion_residual}, providing  details on the forward and reverse processes, velocity-based parameterization, conditioning inputs, and algorithms for training and inference.

\subsection{Forward and Reverse Processes}

We adopt a variance-preserving (VP) setup, with a cosine noise schedule~\cite{nichol2021improved}:
\begin{equation}
\alpha(t) = \cos\left( \frac{\pi}{2} t \right), \quad \sigma(t) = \sin\left( \frac{\pi}{2} t \right),
\quad t \in [0, 1].
\end{equation}
This ensures $\alpha^2(t) + \sigma^2(t) = 1$, with $Z_0 = R$ and $Z_1 \sim \mathcal{N}(0, I)$.

Alternatively, in terms of the log signal-to-noise ratio (log-SNR):
\begin{equation}
\lambda(t) = \log \left( \frac{\alpha^2(t)}{\sigma^2(t)} \right),
\end{equation}
we have
\begin{equation}
\lambda(t) = -2 \log\left( \tan\left( \frac{\pi}{2} t \right) \right).
\end{equation}

For the reverse process, FLEX uses DDIM for deterministic sampling trajectories.

\subsection{Loss Parameterization}

Rather than directly estimating the score or noise, FLEX adopts a velocity prediction formulation~\cite{salimans2022progressive}.  
Given the forward perturbation:
\[
Z_t = \alpha(t) R + \sigma(t) \epsilon, \quad \epsilon \sim \mathcal{N}(0, I),
\]
the velocity is defined as:
\begin{equation}
v(t, R) = \alpha(t) \epsilon - \sigma(t) R.
\end{equation}

A neural network $v_\theta(t, Z_t, C)$ predicts the velocity, and the training loss becomes:
\begin{equation}
\mathcal{L}_{\text{vel}}(\theta) = \mathbb{E}_{t, R, \epsilon} \left[ \left\| v(t, R) - v_\theta(t, Z_t, C) \right\|_2^2 \right],
\end{equation}
where $Z_t$ is obtained from $R$ and $\epsilon$ using the forward process.

This parameterization is used to improve the training stability and sample quality. See Sec. \ref{app_theory_velocitymodel} for a theoretical study of the velocity parametrization.

\subsection{Context Conditioning}

Each sample is conditioned on auxiliary variables to guide generation.  
The context vector $C$ includes:
\begin{itemize}
    \item $\mathbf{c}_{\text{snapshot}}$: low-resolution input or past states,
    \item $\mathbf{c}_{\text{Re}}$: Reynolds number,
    \item $\mathbf{c}_{\text{step}}$: forecasting step index (optional),
    \item $\mathbf{c}_{\text{upsample}}$: super-resolution factor (optional).
\end{itemize}

These features are embedded and provided as input to the neural network at each diffusion step. In practice, we train only with a single step size ($s=1$), and super-resolution factor ($\times 8$), since we observed that this yields the best performance.

\subsection{Algorithms}

The full training and inference procedures for our velocity-based diffusion model are summarized in Algorithms~\ref{alg:training} and~\ref{alg:sampling}, respectively.

\begin{algorithm}[!h]
\caption{Velocity-Based Diffusion Training}
\label{alg:training}
\begin{algorithmic}[1]
\REQUIRE Residual dataset $\mathcal{D} = \{ R \}$, conditioning information $C$, neural network $v_\theta(\cdot)$, noise schedule $(\alpha(t), \sigma(t))$
\FOR{number of training iterations}
    \STATE Sample residual $R \sim \mathcal{D}$ 
    \STATE Sample diffusion time $t \sim \mathcal{U}(0, 1)$
    \STATE Sample noise $\epsilon \sim \mathcal{N}(0, I)$
    \STATE Compute noisy residual: $Z_t = \alpha(t) R + \sigma(t) \epsilon$
    \STATE Compute target velocity: $v(t, R) = \alpha(t) \epsilon - \sigma(t) R$
    \STATE Predict velocity: $\hat{v} = v_\theta(t, Z_t, C)$ 
    \STATE Update parameters $\theta$ using gradient of loss:
    $\mathcal{L}_{\text{vel}} = \left\| v(t, R) - \hat{v} \right\|_2^2$
\ENDFOR
\end{algorithmic}
\end{algorithm}

\begin{algorithm}[!h]
\caption{DDIM Sampling with Velocity Prediction}
\label{alg:sampling}
\begin{algorithmic}[1]
\REQUIRE Neural network $v_\theta(\cdot)$, initial noise $Z_1 \sim \mathcal{N}(0, I)$, conditioning information $C$, noise schedule $(\alpha(t), \sigma(t))$, time steps $\{t_i\}_{i=0}^N$ with $t_0 = 1$, $t_N = 0$
\STATE Initialize $Z_{t_0} \gets Z_1$
\FOR{$i = 1$ \TO $N$}
    \STATE Predict: velocity $\hat{v} = v_\theta(t_{i-1}, Z_{t_{i-1}}, C)$
    \STATE Compute:
    $
    \text{mean} = \alpha(t_{i-1}) Z_{t_{i-1}} - \sigma(t_{i-1}) \hat{v} \quad \text{and} \quad \epsilon = \alpha(t_{i-1}) \hat{v} + \sigma(t_{i-1}) Z_{t_{i-1}}$
    \STATE Update: sample
    $Z_{t_i} = \alpha(t_i) \cdot \text{mean} + \sigma(t_i) \cdot \epsilon$

\ENDFOR
\RETURN Final sample $Z_{t_N}$
\end{algorithmic}
\end{algorithm}

\section{Training and Architecture Details}\label{app:arch}

\subsection{Training Setup}
All FLEX models are trained using the Lion optimizer with a base learning rate of $1\mathrm{e}{-5}$ and a cosine learning rate schedule. We train each  model for 200 epochs using 16 NVIDIA A100 GPUs (80GB each), with a per-GPU batch size of 8 (global batch size of 128). On average, training takes approximately two days for the medium model with 240M parameters.

Instead of the commonly used $\ell_2$ loss for diffusion models, we use an $\ell_1$ loss between predicted and target residuals, which we find improves accuracy, especially for high-frequency fluid structures. Models are trained using 256$\times$256 patches extracted from high-resolution simulation data. No data augmentation is applied.

\paragraph{Exponential Moving Average (EMA).} As is standard in diffusion models, FLEX maintains an EMA of the model parameters during training, with a decay rate of $0.999$.  All final evaluations use the EMA-averaged model weights to improve sample stability and quality. 

\paragraph{Data Normalization.}
We use a simple, fixed normalization strategy for all training and evaluation. Each input field is normalized by subtracting a mean of 0 and dividing by a standard deviation of 5.457 calculated from the average residuals across all Reynold Numbers used in this work. That is, given an input $x$, we apply $x \mapsto (x - \text{mean})/\text{std}$ before feeding it to the model, and invert this operation during evaluation. This normalization is applied consistently across tasks and Reynolds numbers. We fine-tune standard deviation for zero-shot tasks.

\paragraph{Latent Alignment via Contrastive Learning.}

When training the multi-task model, to further align the latent spaces of both tasks, FLEX uses a contrastive loss inspired by CLIP~\cite{radford2021learning}.  Specifically, we apply a symmetric InfoNCE loss between the bottleneck latent features $h_{\text{SR}}$ and $h_{\text{FC}}$ produced by the task-specific encoders:
\begin{equation}
\mathcal{L}_{\text{CLIP}} = \frac{1}{2} \left( \text{InfoNCE}(h_{\text{SR}}, h_{\text{FC}}) + \text{InfoNCE}(h_{\text{FC}}, h_{\text{SR}}) \right),
\end{equation}
which encourages each latent feature to retrieve its corresponding paired representation.

In addition to aligning numeric latent codes, we also experimented with incorporating semantic text embeddings. Given dataset metadata such as Reynolds number, resolution factor, flow modality, and time index, we constructed natural language descriptions of the form:
\begin{quote}
\small
"This is a \{data description\}, with a Reynolds number of \{Re\}, at time step \{index\}."
\end{quote}
We encoded these prompts using a pretrained text encoder and applied CLIP loss to align the text embeddings with the corresponding numeric latent representations.
Although both numeric and text-based alignments did not affect performance, we did not observe any improvements. We hypothesize that alignment may become more beneficial when training across more diverse scientific datasets.

\subsection{Model Scaling}
We train three model variants: \textit{FLEX-S}, \textit{FLEX-M}, and \textit{FLEX-L}, whose architectural differences are detailed in Table~\ref{tab:flex_model_configs}. All models use the same patch size and conditioning structure. The medium configuration is used in all main experiments unless otherwise specified.

\begin{table}[h]
\centering
\caption{FLEX architectural configurations by model size. All models use patch size 1, LayerNorm, and dropout rate of 0.1 for attention and MLP layers.}
\label{tab:flex_model_configs}
\vspace{+0.3cm}
\begin{tabular}{lccc}
\toprule
\textbf{Component} & \textbf{Small} & \textbf{Medium} & \textbf{Large} \\
\midrule
Encoder channels & [64, 128, 128, 256] & [64, 128, 256, 512] & [128, 256, 512, 1152] \\
Encoder blocks & [2, 3, 3, 3] & [2, 3, 3, 4] & [2, 3, 3, 3] \\
Decoder channels & same as encoder & same as encoder & same as encoder \\
Decoder blocks & [2, 3, 3, 3] & [2, 3, 3, 4] & [2, 3, 3, 3] \\
ViT depth & 13 & 13 & 21 \\
ViT heads & 4 & 8 & 16 \\
\midrule
Num. of parm. single-task & 50M & 200M & - \\
Num. of parm. multi-task & 58M & 240M & 1000M \\
\bottomrule
\end{tabular}
\end{table}

\section{Data and Metrics}
\label{sec:data_details}

\subsection{Naviar Stokes Kraichnan Turbulence}  

We follow SuperBench~\cite{ren2025superbench} and generate data using direct numerical simulations (DNS) of the two-dimensional incompressible Navier–Stokes (NS) equations:
\begin{equation}\label{eq:ns_eqn}
\nabla \cdot \mathbf{u} = 0,
\quad
\frac{\partial \mathbf{u}}{\partial t} + (\mathbf{u} \cdot \nabla) \mathbf{u} =
-\frac{1}{\rho} \nabla p + \nu \nabla^2 \mathbf{u},
\end{equation}
where $\mathbf{u} = (u,v)$ represents the velocity field, $p$ is the pressure, $\rho$ the density, and $\nu$ the kinematic viscosity. While these equations form the foundation for modeling fluid flow, they become computationally demanding in turbulent regimes due to the emergence of chaotic, multiscale dynamics.

Our setup is based on the Kraichnan model for two-dimensional turbulence, simulated in a doubly periodic domain $[0, 2\pi]^2$; for details see~\cite{pawar2023frame}. The spatial grid is discretized with $2048^2$ points, and the simulations employ a second-order energy-conserving Arakawa scheme~\cite{arakawa1997computational} to evaluate the nonlinear Jacobian, coupled with a second-order finite-difference approximation for the vorticity Laplacian. We simulate flows at Reynolds numbers ranging from $1000$ to $36{,}000$, using two distinct random initializations for each setting to capture variability.
As Reynolds numbers increase, the flow fields exhibit richer 2D turbulence, with a wider spectrum of interacting eddies and stronger nonlinearities. This makes the data set an interesting testbed for assessing generative models such as FLEX, which must accurately reproduce both local structure and global flow evolution over time.

Each fluid flow snapshot is $2048 \times 2048$ in spatial resolution and is captured at a time step of $\Delta t = 0.0005$ seconds over a total simulation duration of $0.75$ seconds, resulting in 1,500 snapshots per simulation. For each Reynolds number, we generate two simulations using distinct random initial conditions, resulting in approximately 2.3 TB of total data.

\paragraph{Training, Validation, and Test Sets.}
We train on a subset of Reynolds numbers ${2\text{k}, 4\text{k}, 8\text{k}, 16\text{k}, 23\text{k}}$ using data from a single initial condition per $Re$. During training, we used only snapshots with even index ${0, 2, 4, \dots}$, reserving the snapshots with odd index for validation.

We evaluate model generalization under two test scenarios:
\begin{enumerate}
\item \textbf{New initial conditions} at the same Reynolds numbers.
\item \textbf{New initial condition with unseen Reynolds numbers}, including $1\text{k}, 12\text{k}, 24\text{k}, 36\text{k}$.
\end{enumerate}

For testing, we used a total of 70 snapshots per simulation, resulting in approximately 4,480 patches of size $256 \times 256$ per Reynolds number. We use distinct sampling schedules for each task. For super-resolution, we evaluate on every 8th snapshot from the test rollout (i.e., frames ${0, 8, 16, \dots, 560}$). For forecasting, we used a later segment of the simulation, selecting every 12th snapshot starting from snapshots 600 (i.e., ${600, 612, \dots, 1160}$), to ensure temporal separation from the training window.

\textit{Note that both the prediction tasks tend to become easier at later time steps. This is because the initially randomized velocity field undergoes a transient period before reaching a statistically steady state. During this initial phase, the flow exhibits rapidly evolving and less structured behavior that can be difficult for models to learn or predict. As the simulation progresses, the turbulence becomes more fully developed and statistically stable, with coherent multiscale structures that are easier to capture and reconstruct. Consequently, later snapshots of the simulation often exhibit smoother dynamics and more predictable patterns, making the learning problem more tractable.}

\subsection{PDEBench Inhomogeneous Navier–Stokes}

We use a dataset from PDEBench~\cite{takamoto2022pdebench}, which consists of simulations of the two-dimensional, incompressible, inhomogeneous Navier–Stokes equations with Dirichlet boundary conditions. The governing equation includes an external forcing term $\mathbf{u}$:
\begin{equation}
\rho\left(\frac{\partial \mathbf{v}}{\partial t} + \mathbf{v} \cdot \nabla \mathbf{v}\right) = -\nabla p + \eta \nabla^2 \mathbf{v} + \mathbf{u}, \quad \nabla \cdot \mathbf{v} = 0,
\end{equation}
where $\mathbf{v}$ is the velocity field, $p$ is pressure, $\rho$ is density, $\eta$ is viscosity, and $\mathbf{u}$ is a spatially varying forcing field. The domain is $\Omega = [0, 1]^2$, and the velocity is reduced to zero at the boundaries.

The initial conditions $\mathbf{v}_0$ and the forcing fields $\mathbf{u}$ are sampled from isotropic Gaussian random fields with truncated power-law spectra. 

This setup introduces spatial heterogeneity and non-periodic boundary conditions, posing a different modeling challenge compared to the periodically forced Navier–Stokes Kraichnan turbulence (NSKT) described earlier. Although both data sets are based on the Navier–Stokes framework, they differ in boundary treatment, forcing structure, and physical behavior, providing a complementary testbed for evaluating generalization across PDE regimes.

\paragraph{Test Set.}
We use a simulation from the PDEBench~\cite{takamoto2022pdebench} repository. Each snapshot has spatial dimensions of $512 \times 512$. For evaluation, we used 40 temporally spaced snapshots starting from timestep 600, using every 12th frame (i.e., ${600, 612, \dots, 980}$). This setup mirrors the forecasting schedule used in the NSKT evaluations and ensures temporal separation from any training dynamics.

\subsection{Weather Data}

We used a subset of the ERA5 reanalysis dataset from the SuperBench benchmark~\cite{ren2025superbench}, provided by the European Center for Medium-Range Weather Forecasts (ECMWF). ERA5 provides global, hourly estimates of atmospheric variables at high spatial resolution ($0.25^\circ$, or approximately 25 km). For our experiments, we focus on the 2-meter surface temperature field, sampled daily at 00:00 UTC. Each snapshot is represented on a $720 \times 1440$ latitude–longitude grid. This variable is widely used in climate studies and weather prediction due to its relevance for extreme heat events and long-term climate signals.

\paragraph{Training, Validation, and Test Sets.}
Following the SuperBench setup~\cite{ren2025superbench}, we use data from 1979 to 2015 for training, and reserve data for 2016 and 2017 for testing. Validation is performed on held-out days from the training period. All experiments use a single channel input corresponding to the temperature field. Additional channels (e.g., wind or water vapor) are available in SuperBench but are not used in this work.

\subsection{Evaluation Metrics}
We assess model quality using:
\begin{itemize}
    \item \textbf{Relative Frobenius norm error (RFNE):}
    \begin{equation}
        \mathrm{RFNE} = \frac{\|\hat{\mathbf{X}} - \mathbf{X}\|_F}{\|\mathbf{X}\|_F},
    \end{equation}
    where $\hat{\mathbf{X}}$ is the predicted field and $\mathbf{X}$ is the ground-truth field.
    \item \textbf{Pearson correlation coefficient (PCC):}
    \begin{equation}
        \mathrm{PCC} = \frac{\sum_i (\hat{x}_i - \bar{\hat{x}})(x_i - \bar{x})}{\sqrt{\sum_i (\hat{x}_i - \bar{\hat{x}})^2 \sum_i (x_i - \bar{x})^2}},
    \end{equation}
    where $\bar{\hat{x}}$ and $\bar{x}$ are the mean values of the predicted and true fields, respectively. This metric emphasizes alignment of coherent structures in turbulent flow.

    \item  \textbf{Computing the Vorticity Energy Spectrum.} To assess the physical fidelity of predicted flows, we compute the energy spectrum of the vorticity field. Although energy spectra are commonly derived from velocity fields, vorticity offers a more direct measure of rotational and eddy-scale dynamics in two-dimensional turbulence. Given a vorticity field $\omega(x, y)$, we first apply a two-dimensional Fast Fourier Transform (FFT) to obtain its spectral representation $\hat{\omega}(k_x, k_y)$, where $(k_x, k_y)$ are the discrete wavevector components. We then compute the radially averaged energy spectrum $E_\omega(k)$ by aggregating the power $\pi |\hat{\omega}(k_x, k_y)|^2 / [(N_x N_y)^2 \sqrt{k_x^2 + k_y^2}]$ across all spectral modes with magnitude $k = \sqrt{k_x^2 + k_y^2}$, using a binning procedure to average over narrow radial bands. Low wavenumbers correspond to large-scale coherent structures, while high wavenumbers capture fine-grained turbulent features. As vorticity and velocity are analytically related in 2D, this spectrum provides a reliable multiscale diagnostic for evaluating how well super-resolved outputs preserve the physical structure of turbulent flows.
\end{itemize}

\section{Additional Experimental Results}\label{app:results}

Figures~\ref{fig:forecast_vorticity}, \ref{fig:forecast_velocity}, and~\ref{fig:forecast_pdebench} show qualitative examples of long-horizon forecasting using FLEX. The first two illustrate predicted vorticity and velocity fields at $Re = 12{,}000$, while the third shows a zero-shot prediction on PDEBench. Across all cases, the model produces visually coherent and physically plausible outputs, even 30–40 steps into the forecast horizon.

\begin{figure}[!b]
    \centering
    \hspace{+0.5cm}\begin{overpic}[width=0.95\linewidth]
    {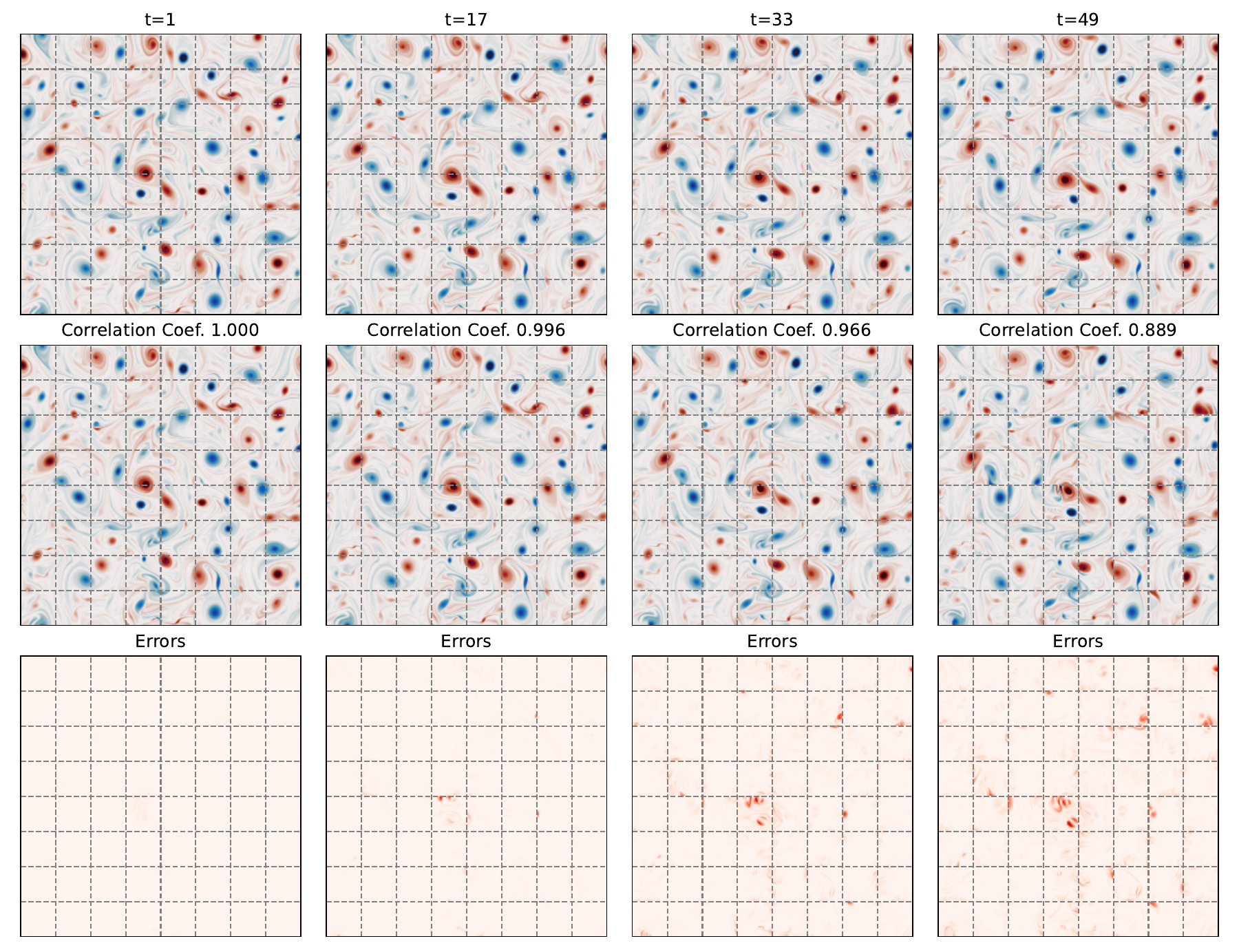}
     \end{overpic}
        \begin{overpic}[width=0.99\linewidth]
    {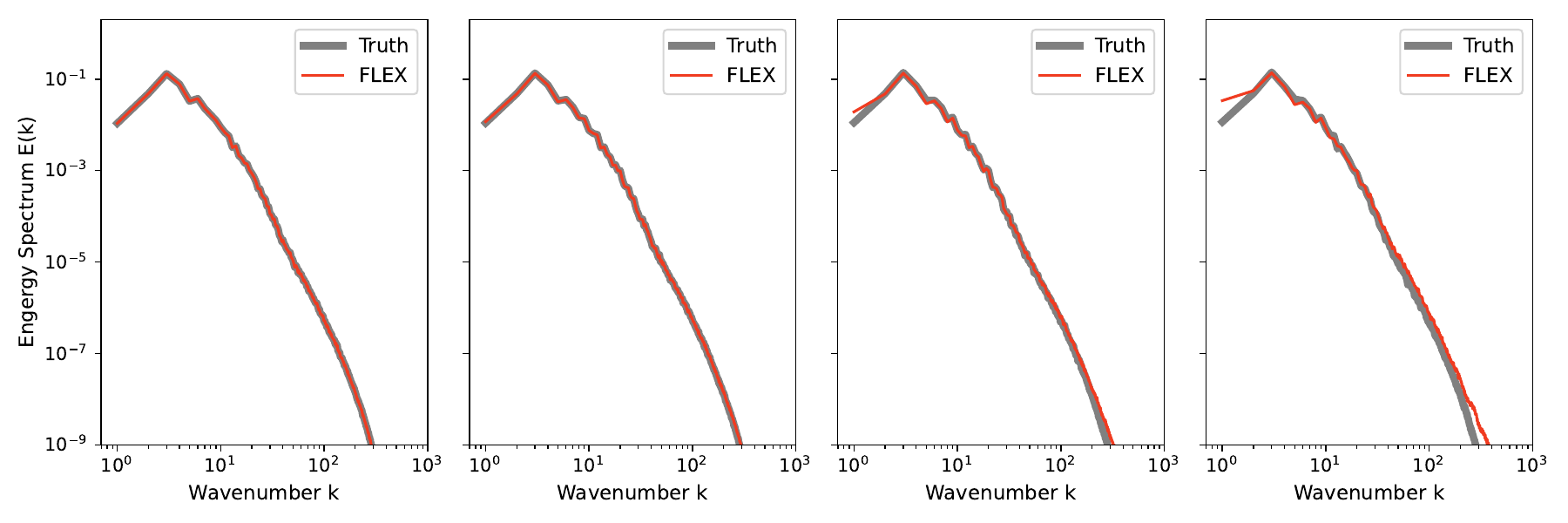}
     \end{overpic}
    \caption{Forecast of vorticity at $Re = 12{,}000$. FLEX maintains coherent vortex structures and preserves fine-scale features up to 40 steps ahead, despite the highly chaotic nature of the flow.}
    \label{fig:forecast_vorticity}
\end{figure}

\begin{figure}[!b]
    \centering
    \hspace{+0.5cm}\begin{overpic}[width=0.95\linewidth]
    {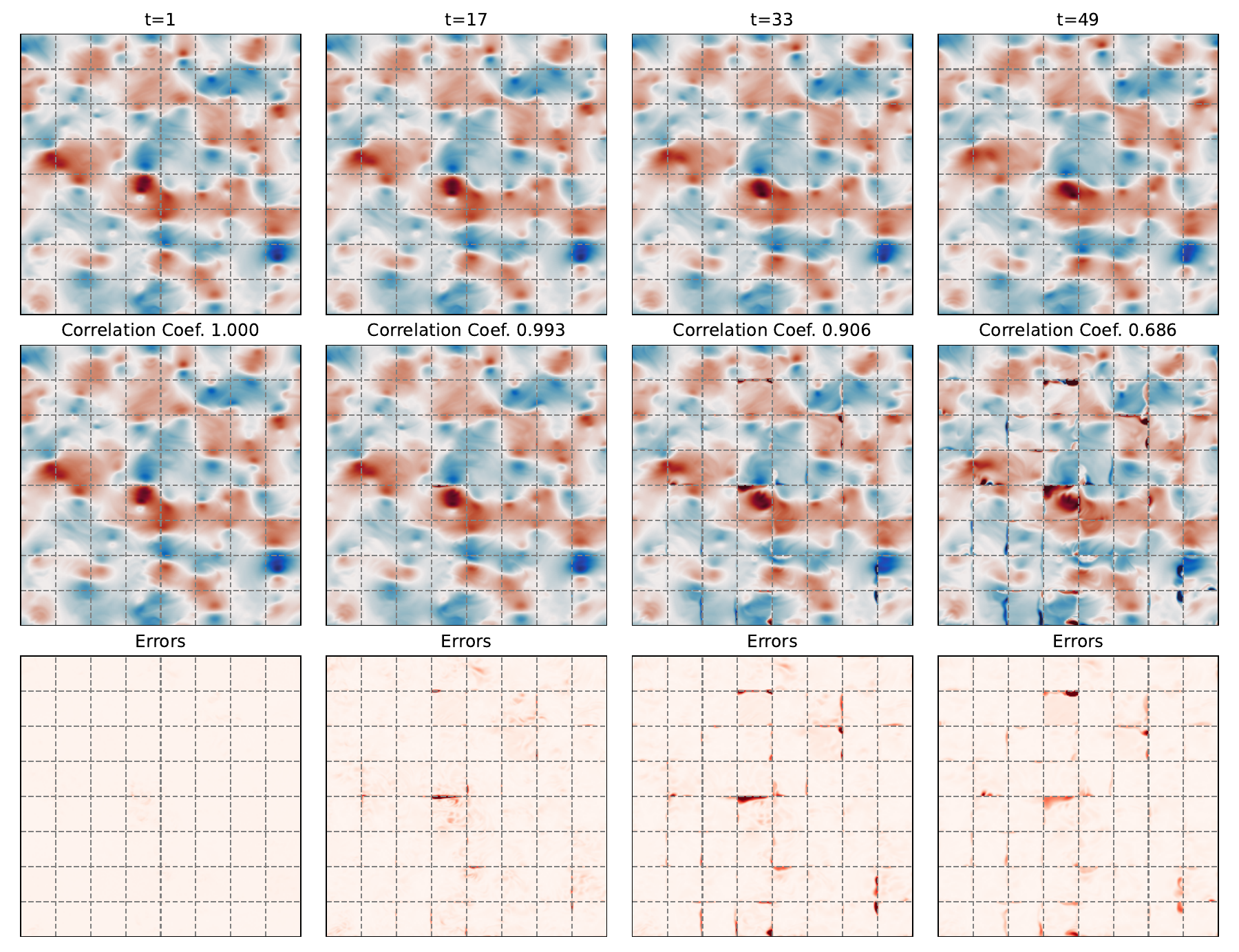}
     \end{overpic}
    \begin{overpic}[width=0.99\linewidth]
    {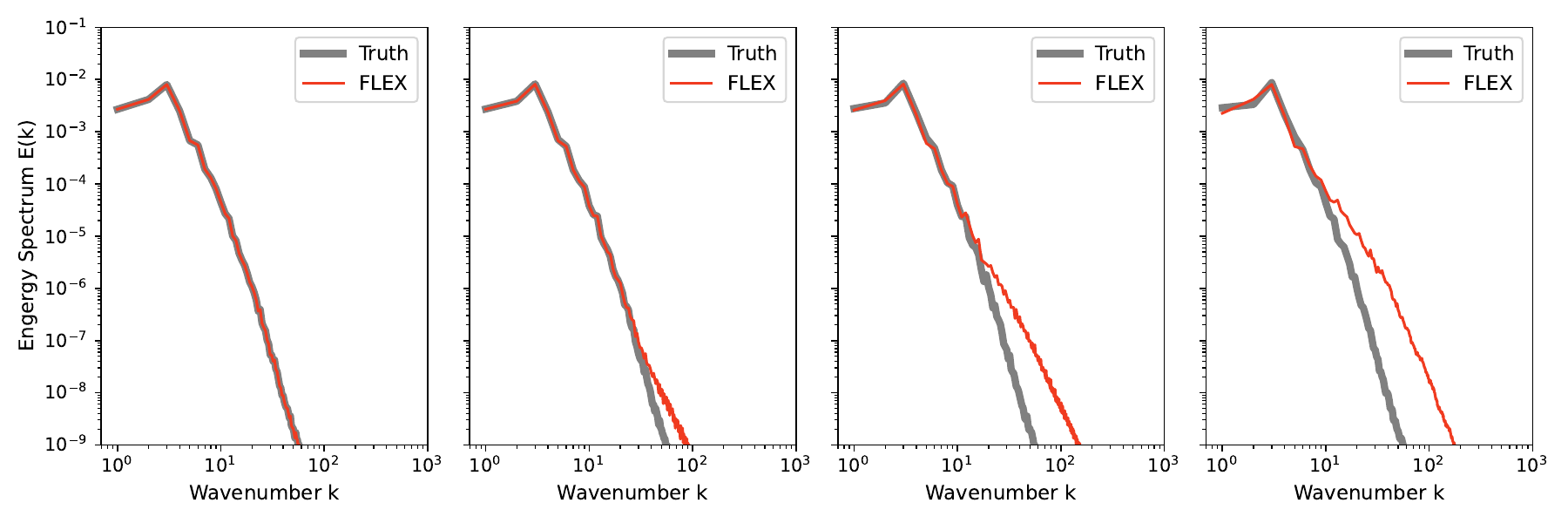}
     \end{overpic}
    \caption{Velocity field forecast at $Re = 12{,}000$, a variable not used during training. FLEX generalizes well across observables, producing physically consistent structures even under cross-variable prediction.}
    \label{fig:forecast_velocity}
\end{figure}

\begin{figure}[!b]
    \centering
    \hspace{+0.5cm}\begin{overpic}[width=0.95\linewidth]
    {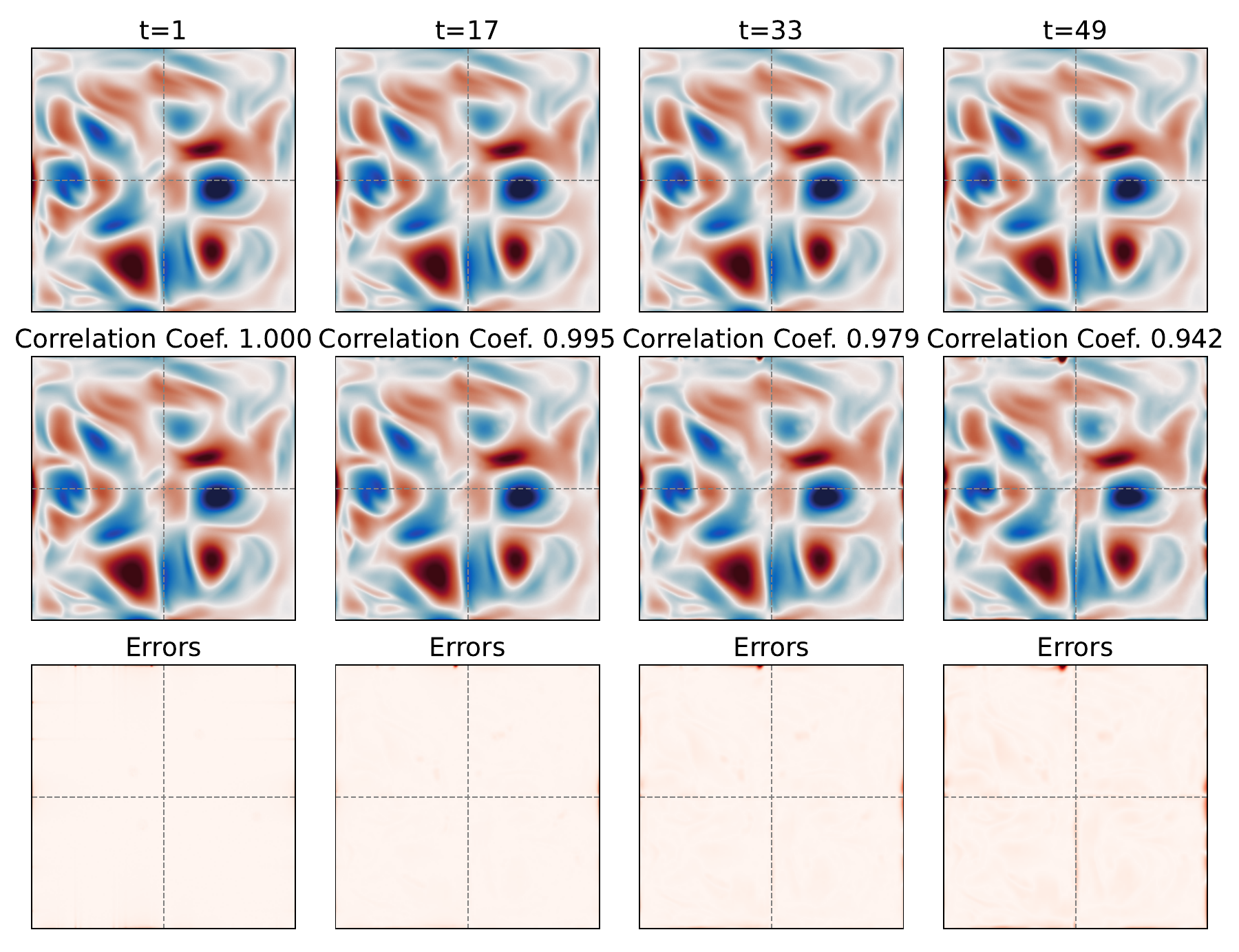}
     \end{overpic}
    \begin{overpic}[width=0.99\linewidth]
    {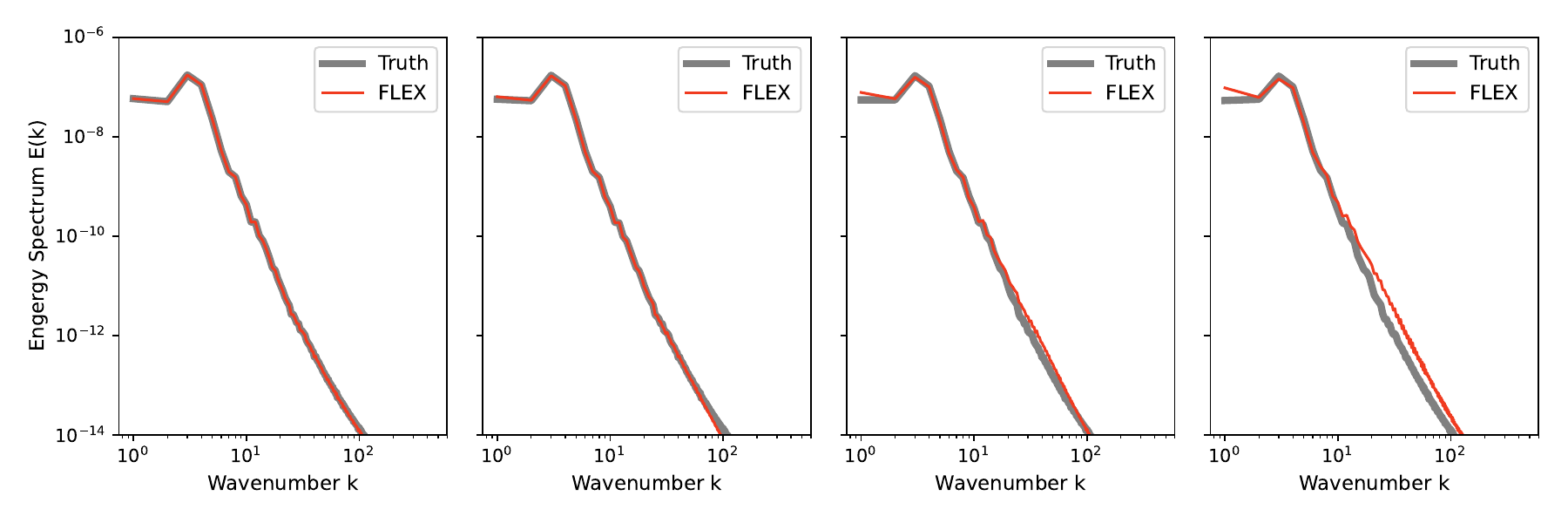}
     \end{overpic}
    \caption{Forecasted vorticity from PDEBench’s inhomogeneous Navier–Stokes data with Dirichlet boundaries. Despite the domain shift and different forcing dynamics, FLEX produces stable and plausible predictions over the full rollout.}
    \label{fig:forecast_pdebench}
\end{figure}

\subsection{Comparison of FLEX and HAT Models on Weather Data}

To assess the broader applicability of FLEX beyond fluid dynamics, we evaluate its performance on a spatiotemporal super-resolution task using global weather data. Specifically, we compare FLEX with the HAT model~\cite{chen2023activating} on 2-meter surface temperature fields from the ERA5 dataset.

Both models are trained on daily temperature snapshots from 1979 to 2015 and evaluated on data from 2016 and 2017. Figure~\ref{fig:climate-sr} shows an example reconstruction from the year 2016. Table~\ref{tab:weather_results} reports performance metrics, including RFNE, structural similarity index (SSIM), and peak signal-to-noise ratio (PSNR).
FLEX outperforms HAT in all metrics, highlighting its ability to recover fine-scale features in complex atmospheric fields.

\begin{table}[!ht]
\caption{Comparison of FLEX and HAT on climate super-resolution. Results are averaged over test snapshots from the ERA5 dataset (2016–2017).}\vspace{+0.3cm}
\label{tab:weather_results}
\centering
\begin{tabular}{lcc}
\toprule
\textbf{Metric} & \textbf{FLEX (Mean)} & \textbf{HAT (Mean)} \\
\midrule
RFNE (\%)  & 5.88  & 6.98  \\
SSIM (\%)  & 95.73  & 94.33 \\
PSNR (dB)  & 36.59  & 35.27 \\
\bottomrule
\end{tabular}
\end{table}

\begin{figure}[!t]
    \centering
    \includegraphics[width=0.8\textwidth]{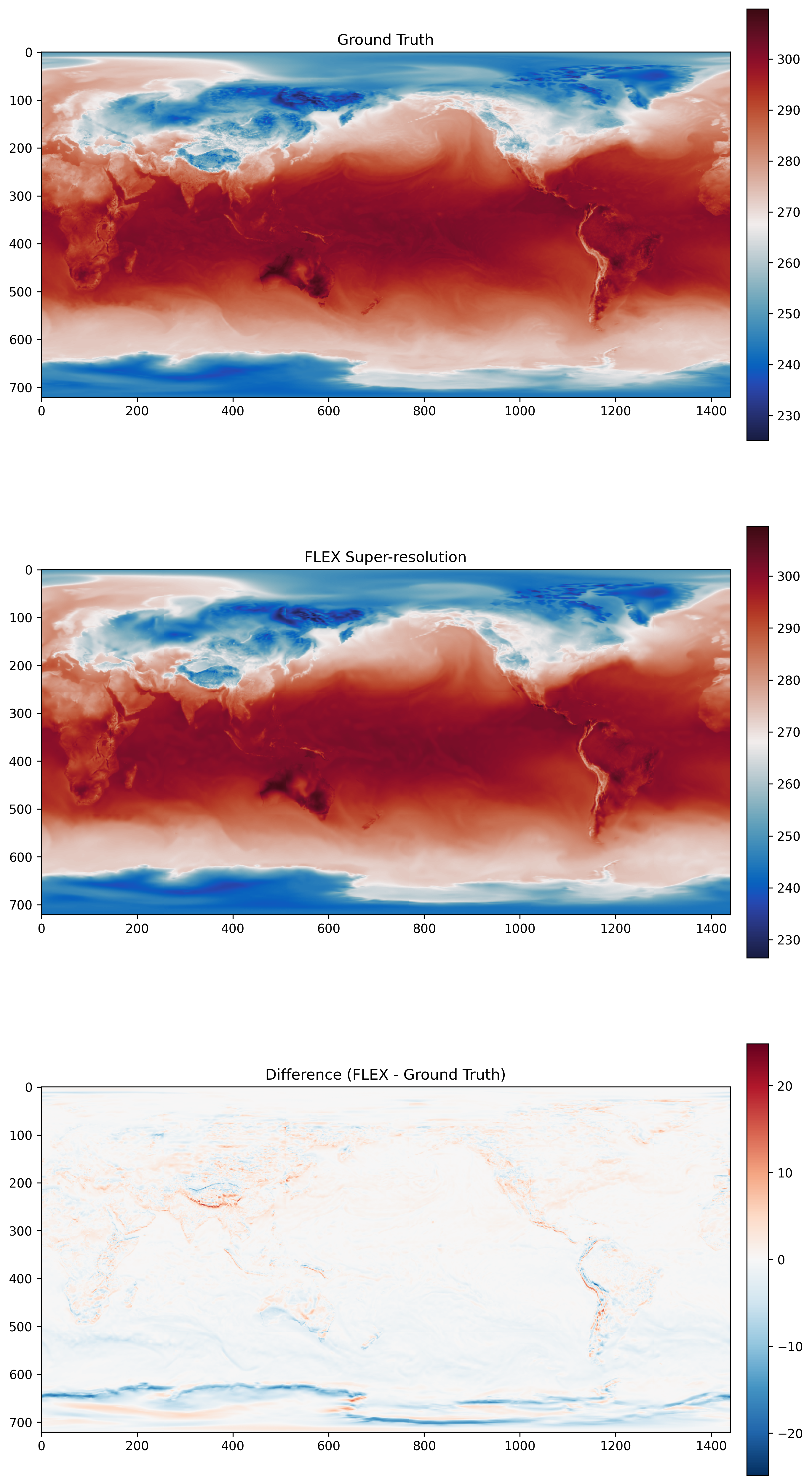} 
    \caption{Comparison of reconstructed temperature fields from the ERA5 dataset. (top) Ground truth high-resolution temperature field, (middle) super-resolved output generated by FLEX, and (bottom) difference map between FLEX and the ground truth. The results show FLEX’s ability to recover fine-scale structures while maintaining low reconstruction error across the spatial domain.}
    \label{fig:climate-sr}
\end{figure}

\end{document}